\documentclass{article}
\usepackage{ijcai24}

\usepackage{times}
\usepackage{soul}
\usepackage{url}
\usepackage[hidelinks]{hyperref}
\usepackage[utf8]{inputenc}
\usepackage[small]{caption}
\usepackage{graphicx}
\usepackage{subfig}
\usepackage{amsmath}
\usepackage{amsthm}
\usepackage{bbm, dsfont}
\usepackage{booktabs}
\usepackage{algorithm}
\usepackage{algorithmic}
\usepackage[switch]{lineno}
\usepackage{amssymb}
\usepackage{makecell}

\usepackage{changes}
\usepackage{lipsum}% <- For dummy text

\newtheorem{theorem}{Theorem}
\newtheorem{definition}{Definition}
\newtheorem{lemma}{Lemma}

\newcommand{\namemodel}{{DHGAK}{}}

\title{Deep Hierarchical Graph Alignment Kernels}

\author{
Shuhao Tang$^1$
\and
Hao Tian$^1$\and
Xiaofeng Cao$^2$\And
Wei Ye$^{1,\ast}$\\
\affiliations
$^1$Tongji University, Shanghai, China\\
$^2$Jilin University, Changchun, China\\
\emails
\{tangsh2022, yew\}@tongji.edu.cn\\
xiaofengcao@jlu.edu.cn
}
\begin{document}

\maketitle

\begingroup\renewcommand\thefootnote{$^\ast$}
\footnotetext{Corresponding author.}
\endgroup

\begin{abstract}
    % In the past decade, machine learning tasks on graph structures have gained widespread attention. Graph kernels, as a crucial technology in the field of graph analysis, primarily adhere to the R-convolution theoretical framework, which decomposes graphs into non-isomorphic substructures. However, these methods tend to isolate the substructures, neglecting the distribution of substructures within the graph and the connections between similar substructures. In this paper, we emphasize the relationships between these substructures, and the embedding of substructures is automatically learned. We introduce a novel framework named HUGE. This framework captures the relationships between substructures through clustering methods and incorporates their distribution in the graph using kernel mean embedding. The feasibility of the HUGE framework is guaranteed through the theoretical analysis of risk deviation and the linear separability in Reproducing Hilbert Kernel Space (RHKS). Experiments on several datasets with many the state-of-the-art graph kernels on graph classification task show that HUGE we proposed achieves the best performance.
    %However, typical tricks often isolate substructures in two aspects: overlooking the distribution of substructures and neglecting the connections between similar ones. Theoretically, the lack of structural implications in the association substructures leads to an incomplete graph topology, 
    
  Typical $\mathcal{R}$-convolution graph kernels invoke the kernel functions that decompose graphs into non-isomorphic substructures and compare them. However, overlooking implicit similarities and topological position information between those substructures limits their performances. In this paper, we introduce Deep Hierarchical Graph Alignment Kernels (DHGAK) to resolve this problem. Specifically, the relational substructures are hierarchically aligned to cluster distributions in their deep embedding space. The substructures belonging to the same cluster are assigned the same feature map in the Reproducing Kernel Hilbert Space (RKHS), where graph feature maps are derived by kernel mean embedding. Theoretical analysis guarantees that DHGAK is positive semi-definite and has linear separability in the RKHS. Comparison with state-of-the-art graph kernels on various benchmark datasets demonstrates the effectiveness and efficiency of DHGAK. The code is available at Github\footnote{\url{https://github.com/EWesternRa/DHGAK}}.

\end{abstract}

\section{Introduction}

Over the past decades, more and more data in studies depicting relationships between objects cannot be simply interpreted as vectors or tables. Instead, they are naturally represented by graphs. In order to capture the inherent information from graph-structured data, many graph-based machine learning algorithms are gaining attraction.

Graph kernels (GKs) \cite{Graphkernels} are conventional methods for measuring the similarity of graphs by kernel method defined on graphs. Most graph kernels are instances of $\mathcal{R}$-convolution theory \cite{R-convolution}. These methods decompose graphs into non-isomorphic substructures (e.g., paths or walks, subgraphs, subtrees) and compare them. Graph similarity is the sum of all the substructure similarities. However, the computation of the substructure similarity is strict, i.e., all non-isomorphic substructures are treated as dissimilar ones, which impedes the precise comparison of graphs. In addition, \(\mathcal{R}\)-convolution graph kernels neglect the topological position information between substructures, which plays a crucial role in the characteristic of the whole graph, e.g., whether or not a drug molecule is effective on a disease.

\begin{figure}
    \centering
    \subfloat[Aspirin]{
    \includegraphics[width=.25\columnwidth]{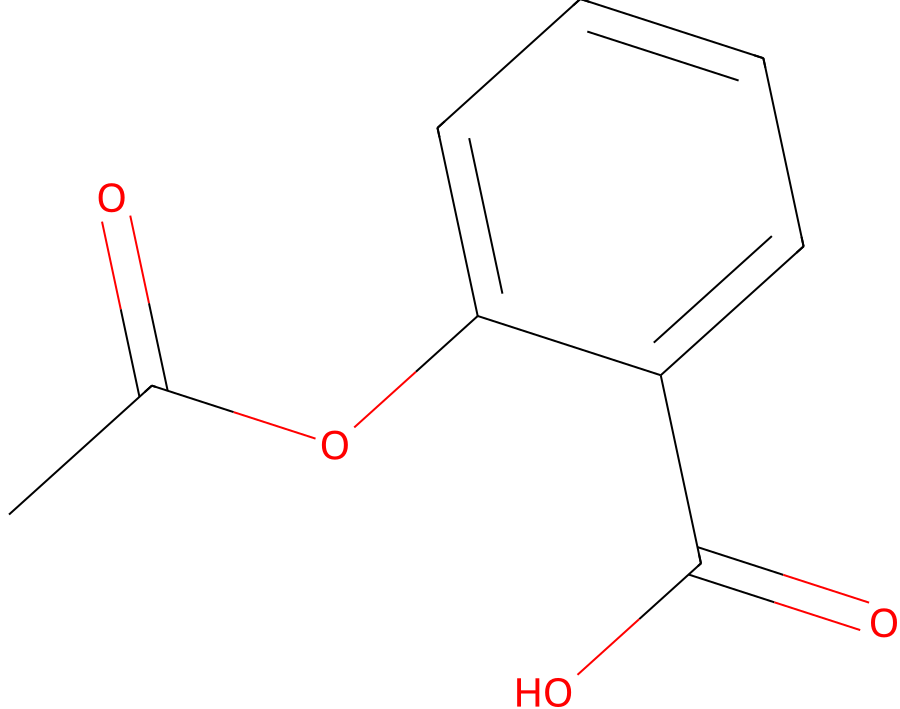}
    \label{fig:aspirin}
    }
    \subfloat[3-Acetoxybenzoic acid]{
    \includegraphics[width=.32\columnwidth]{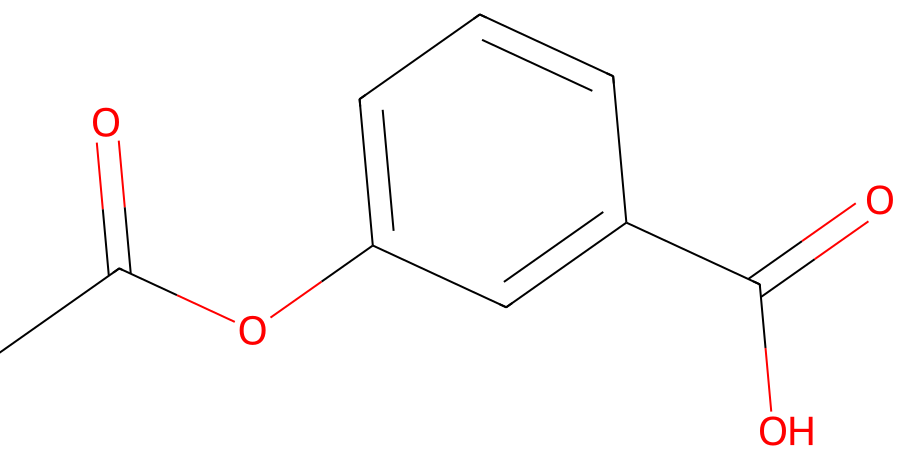}
    \label{fig:3-acetoxybenzoic acid}
    }
    \subfloat[Acetaminophen]{
    \includegraphics[width=.32\columnwidth]{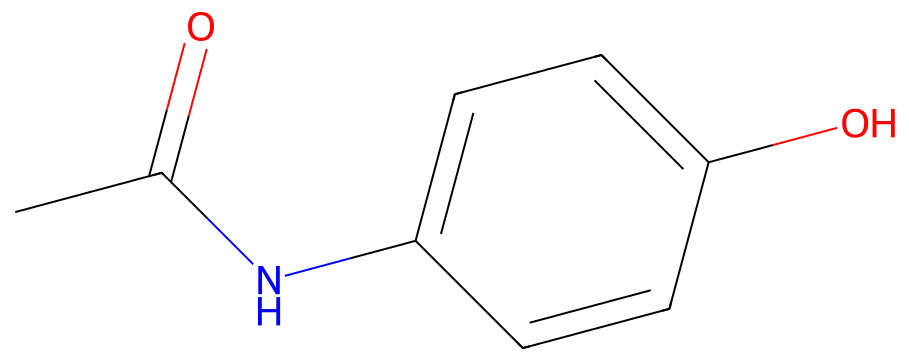}
    \label{fig:Acetaminophen}
    }
    \caption{The structure 2D depictions of Aspirin, Acetaminophen, and 3-Acetyloxybenzoic acid.}
    \label{fig:mols}
\end{figure}

% Furthermore, since non-isomorphic substructures are treated as dissimilar, the frequency of each substructure is independently counted, neglecting their inherent similarities. These issues cause the \(\mathcal{R}\)-convolution graph kernels not precisely aware of the distributional information of graph substructures, thus limiting their performances. In addition, \(\mathcal{R}\)-convolution graph kernels neglect the topological position information between substructures, which plays a crucial role in the characteristic of the whole graph, e.g., whether or not a drug molecule is effective on a disease.

We show the 2D representations of three molecules in Figure \ref{fig:mols} to explain the importance of precise substructure similarity comparison and their relative topological positions. 
Aspirin (Figure \ref{fig:mols} (a)) and 3-Acetoxybenzoic acid (Figure \ref{fig:mols} (b)) have the same benzoic and acetyloxy substructures, but the former is commonly used in drugs for the treatment of pain and fever, while the latter is not. The difference in the relative topological positions of these two substructures causes the different molecule characteristics. \(\mathcal{R}\)-convolution graph kernels are difficult to distinguish the two molecules because they are unaware of the substructure topological position information. Aspirin and Acetaminophen (Figure \ref{fig:mols} (c)) have similar substructures acetylamino and acetyloxy (only a little bit of difference in the atoms), they are both used for pain and fever. However, \(\mathcal{R}\)-convolution graph kernels cannot precisely compute the similarity between these similar substructures from them.

% R-convolution kernels have achieved significant success in graph classification; however, they currently face two challenges: \textit{isolation} and \textit{artificiality}. Many R-convolution kernels decompose graphs into non-isomorphic substructures (e.g., walks, paths, subtrees) and count their occurrences in each pair of graphs, aggregating them to define the similarity between the graphs. The issue lies in the fact that these methods treat substructures as isolated. In R-convolution kernel methods, the occurrences of two similar but distinct substructures are independently count, leading to the neglect of similarity between substructures. 

% Furthermore, the occurrences of different substructures serve as an approximation of the original graph structure, but many substructures heavily relies on human's handcrafted. It is hard to give a theoretical basis for whether a certain substructure is valid or not.  

Many graph kernels try to alleviate these problems in different ways. 
WWL \cite{WWL} embeds graph substructures (WL \cite{WL-subtree_GK} subtree patterns) into continuous vectors and compares their distributions by leveraging the Wasserstein distance.
%, considering the similarity between non-isomorphic WL subtree patterns.
% Since the computation of Wasserstein distance needs to find the optimal transport involving each pair of WL subtree patterns, 
The use of optimal transport in WWL implicitly aligns WL subtree patterns from different graphs, thus addressing the neglect of the topological position information between WL subtree patterns. But WWL has high computational overhead and is not positive semi-definite, since the optimal transport is not guaranteed to satisfy the transitivity. In other words, if $\sigma$ is the alignment between substructures from graphs $G_1$ and $G_2$ and $\pi$ is the alignment between substructures from graphs $G_2$ and $G_3$, the optimal transport cannot guarantee that the alignment between substructures from graphs $G_1$ and $G_3$ is $\sigma\circ\pi$.
GAWL \cite{GAWL} is a positive semi-definite graph alignment kernel that explicitly aligns WL subtree patterns by comparing their labels refined by the Weisfeiler-Leman graph isomorphism test \cite{WL-test}. However, with more iterations of the WL-relabeling process, it is more difficult to align since all the WL subtree patterns tend to have different labels.
% GAWL \cite{GAWL} replaces the injective relabeling function in the original WL to gain a more gradual refinement of labels, which tries to capture similar substructure information.

% \textcolor{blue}{HTAK \cite{bai2022hierarchical} hierarchically aligns the neighborhoods of each pair of nodes in different graphs in an entropy perspective, which can roughly capture the structural information in graphs.}
% \todoo{Talk about why using graph alignment. Reason: For example, do not need to compute the similarity between each pair of graph substructures.}
% Graph Neural Networks (GNNs) \cite{GNNsReview}  recently shows rich expressiveness on several graph tasks, but GNNs lose the easy-to-use features of GKs, and face other problems such as over-smoothing.

To mitigate the above problems, we propose a new graph kernel called Deep Hierarchical Graph Alignment Kernels (DHGAK). 
Specifically, we define a new graph substructure named $b$-width $h$-hop slice around a node to capture the structural information. Then, we embed each slice into a deep vector space by Natural Language Models (NLM) for precise comparison. In the deep embedding space, we propose a Deep Alignment Kernel (DAK) that aligns slices from the same or different graphs by computing the expectation of the cluster assignments of the adopted clustering methods. Slices from the same cluster are aligned while those from different clusters are unaligned. We prove that the alignment by clustering in DAK satisfies the transitivity. 
% Instead of computing similarities between each pair of Slices with high time complexity, we propose a novel kernel, namely Graph Alignment Kernel (DAK). DAK aligns similar Slices to the same cluster by considering the expectation over the given clustering methods. 
By assuming the feature maps of DAK in one graph are i.i.d. (independent and identically distributed) samples of an unknown distribution, we adopt the kernel mean embedding \cite{kernel-mean-embedding} to construct Deep Graph Alignment Kernel (DGAK) from a distributional perspective. With the summation of DGAK on each hop of $b$-width slice, we exploit hierarchical graph structural information and finally get our Deep Hierarchical Graph Alignment Kernels (DHGAK). All the three kernels, i.e., DAK, DGAK, and DHGAK, are positive semi-definite. It is theoretically demonstrated that there exists a family of clustering methods derived from which the feature map of our proposed {\namemodel} is linearly separable for any dataset in the Reproducing Kernel Hilbert Space (RKHS).

Our contributions are summarized as follows:

\begin{itemize}

    \item We propose a generalized graph alignment kernel DHGAK that can use any NLM to generate deep embeddings of substructures and any clustering methods to align substructures. DHGAK can more precisely compute graph similarity compared to conventional $\mathcal{R}$-convolution graph kernels by not only taking into account the distributions of substructures but also their topological position information in graphs.
    
    \item We prove that DHGAK is positive semi-definite and satisfies alignment transitivity. Theoretical analysis proves that the feature map of {\namemodel} produced by certain clustering methods is linearly separable in the RKHS.
    
    \item We compare DHGAK with state-of-the-art graph kernels on various benchmark datasets and the experimental results show that {\namemodel} has the best classification accuracy on most datasets.
    
\end{itemize}

\section{Related Work}

Weisfeiler-Lehman graph kernel (WL) \cite{WL-subtree_GK} quantifies graph similarity by calculating the number of common labels (subtree patterns) between two graphs at each iteration of the WL-relabeling process (Weisfeiler-Lehman graph isomorphism test \cite{WL-test}), which however diverges too fast and leads to coarse graph similarity. 
% Many studies have extended WL or incorporated it into the design of new graph kernels. 
Wasserstein Weisfeiler-Lehman graph kernel (WWL) \cite{WWL} extends WL to graphs with continuous node attributes. It uses Wasserstein distance to measure the distance between graphs but the cost of computation is very expensive. 
% Graph Filtration Kernels \cite{GraphFiltrationKernels(FWL)} consider the occurrence of graph features sequentially by tracking the evolving progress of the graph edges and comparing the WL subtree pattern distribution in the process. 
Gradual Weisfeiler-Lehman (GWL) \cite{GWL} introduces a framework that allows a slower divergence of the WL relabeling process.
% GWL assigns new labels to the neighborhood of a node rather than to a single node while the label assignment method is injective. 
GAWL \cite{GAWL} initiates the alignment of nodes across distinct graphs based on their shared labels, determined by the WL-relabeling process. 
% Subsequently, the adjacency matrix of each graph undergoes reordering through the application of a permutation matrix derived from the established node alignment, and then a linear kernel on the aligned adjacency matrices is obtained. 
It becomes difficult for node alignment if the iteration number of the WL-relabeling process increases.

Shortest Path graph kernels (SP) \cite{Shortest-path-GK} quantify graph similarity by evaluating the shortest paths that have the same labels of the source and sink nodes and the same length between all node pairs within two graphs. 
% The key insight of SP is that the more number of shortest paths of the same representation in the two graphs, the more similar they are. 
However, the shortest path representation is too simple and may lose information. Tree++ \cite{Tree++} and MWSP~\cite{ye2023multi} extract the multi-scale shortest-path features in the BFS tree rooted at each node of graphs. 
% RetGK \cite{RetGK} is a graph kernel designed for both attributed and non-attributed graphs. 
RetGK \cite{RetGK} represents each graph as a set of multidimensional vectors consisting of return probability resulting from random walks starting and ending at identical nodes. The embedding in the RKHS is achieved via Maximum Mean Discrepancy (MMD) \cite{MMD}.

Beyond the graph kernels mentioned above, assignment kernels have attracted the attention of scholars. The main idea is to compute an optimal matching between the substructures of each graph pair, reflecting the graph similarity.
% However, most assignment kernels suffer from an issue: the assignment yields indefinite functions \cite{vertoptimal}, thus requiring more effort to ensure their positive semi-definiteness. 
WL-OA~\cite{WL-OA} first designs some base kernels to generate hierarchies and then computes the optimal assignment kernels. Such kernels are guaranteed to be positive semi-definite. PM \cite{PM} 
% first represents each graph as a set of node embeddings, i.e., the eigenvectors of its adjacency matrix. Then two novel graph kernels are designed to measure the graph similarity. The first one computes the distance between two graph embeddings with the Earth Mover’s Distance \cite{rubner2000earth}, while the kernel matrix is not always positive semi-definite. The second one 
adopts the Pyramid Match kernel \cite{lazebnik2006beyond} to find an approximate correspondence between two graph embeddings (the eigenvectors of graph adjacency matrix). HTAK \cite{bai2022hierarchical} computes the $H$-hierarchical prototype representations and hierarchically aligns the nodes of each graph to their different $h$-level ($1 \leq h \leq H$) prototypes.
% Then the positive semi-definite graph kernel is derived by counting the number of aligned node pairs. 
% However, the hierarchical prototype generation depends on k-means, which relies on predefined clustering numbers.
However, HTAK can only work for un-attributed graphs.

% Many of the above methods are based on $\mathcal{R}$-convolution theory \cite{R-convolution}. 
Other graph kernels are as follows. 
% Deep Graph Kernels (DGK) \cite{DGKandIBDM_Reddit} use natural language processing techniques word2vec \cite{word2vec} to embed graph substructures and integrate the similarity matrix between graph substructures into the Gram matrix of graph kernels but the matrix computation is time-consuming if the number of substructures is huge. 
Graph Neural Tangent Kernels (GNTKs) \cite{GNTK}, which correspond to infinite-width multilayer Graph Neural Networks (GNNs) trained by gradient descent, have the same rich expressiveness as GNNs and inherit the advantages of graph kernels. Graph Quantum Neural Tangent Kernel (GraphQNTK) \cite{GraphQNTK}, a novel quantum graph learning model, is designed to capture structural information and exploit quantum parallelism for computational efficiency. The model approximates the underlying GNTKs by introducing a multi-head quantum attention mechanism. Although both GNTK and GraphQNTK exploit the high expressiveness of GNNs, they suffer from the overparameterization and overfitting issues of GNNs on small graphs.

\begin{figure*}[!htbp]
    \centering
    \includegraphics[width=0.95\linewidth]{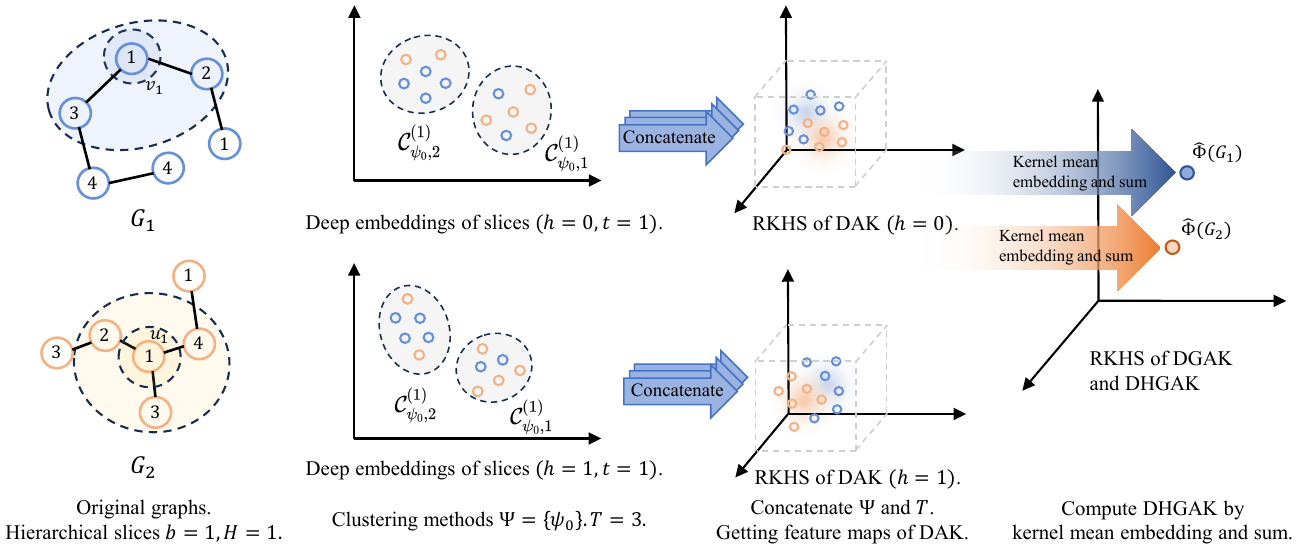}
    \caption{The framework of DHGAK is presented with slice width \(b = 1\), maximum hop \(H = 1\), clustering method set \(\Psi = \{\psi_0\}\), and experiment times \(T = 3\). For convenience, we illustrate our method from the perspective of space transformation. First, we construct the hierarchical \(b\)-width \(h\)-hop slice for each node in \(G_1\) and \(G_2\), where \(b = 1\) and \(h\) is taken from 0 to 1. Next, the encoding method mentioned in Section \ref{sec:SliceEncoding} is used to obtain the encoding sequence of slice \(S_h^b(v)\). The numbers in nodes represent the node labels in the original graphs,  we can get \(S_0^1(v_1) = [1,3,2]\), \(S_1^1(v_1) = [3, 1, 4, 2, 1, 1]\), \(S_0^1(u_1)=[1,2,4,3]\), and \(S_1^1(u_1) = [2, 1, 3, 4, 1, 1, 3, 1]\). Then, the encoding of the slice is embedded into deep embedding space via a Natural Language Model and updated by Equation \ref{eq:node_emb}. Within this deep embedding space, we cluster similar slices and concatenate the cluster indicators for all clustering methods \(\Psi\) and experiment times \(T\) to obtain the feature maps of DAK. \(\mathcal{C}_{\psi_0, i}^{(t)}\) represents the \(i\)-th cluster at the \(t\)-th experiment under clustering method \(\psi_0\in \Psi\). Finally, the feature map of DGAK is the kernel mean embeddings of the DAK feature maps. The feature map of DHGAK is computed as the sum of those of DGAK on slices of different hierarchies.}
    \label{fig:framework}
\end{figure*}

\section{The Model}
\label{sec:TheModel}

{\namemodel} defines graph kernels by hierarchically aligning relational substructures to cluster distributions. The framework of DHGAK is presented in Figure \ref{fig:framework}.
Specifically, Section~3.1 introduces the notations used in this paper; Section 3.2 describes how to encode a slice into a unique sequence and then embed it into a deep vector space via Natural Language Models (NLM); In Section 3.3, the Deep Alignment Kernel (DAK) is defined and the feature map of {\namemodel} is obtained in a distribution perspective; Section 3.4 gives the theoretical analysis of {\namemodel}.

\subsection{Notation}
The notations used in this paper are as follows. We consider undirected and labeled graphs. 
Given a graph dataset \(\mathcal{D}=\{G_1,...,G_N\}\), graph \( G_i = (\mathcal{V}_i, \mathcal{E}_i, l) \), where \(\mathcal{V}_i\) and \(\mathcal{E}_i\) are the sets of nodes and edges of graph \(G_i\), respectively. \(l: \mathcal{V}_i\to \Sigma\) is a function that assigns labels in the label alphabet \(\Sigma\) to nodes. 
% For simplicity, we also use the \(l\) to represent the function that maps nodes from different graphs to their corresponding node labels.
For the graph classification problem, each graph \(G_i\) is assigned a class label $y_i\in\{1,2,\ldots,c\}$ ($c$ is the number of classes).
% An undirected and labeled graph dataset is denote as \(D=\{(G_i, y_i)\}_{i}^{m} = \{((V_i, E_i, l_v), y_i)\}_{i=1}^{m}\).

\subsection{Hierarchical Neighborhood Structure Encoding}

\label{sec:SliceEncoding}
First, we define the \(b\)-width \(h\)-hop slice of a node as follows:

\begin{definition}[\(b\)-width \(h\)-hop Slice of Node]
    Given an undirected graph \(G=(\mathcal{V},\mathcal{E})\), rooted at each node \(v\in \mathcal{V}\), we build a truncated BFS tree \(T_v\) of depth \(r\). All the nodes in \(T_v\) denoted as \(N^r(v)\) and all the leaf nodes (nodes in the \(r\)-th depth) in \(T_v\) denoted as \(N_r(v)\) are defined as the \(r\)-width and the \(r\)-hop neighbors of \(v\), respectively. Then, the \(b\)-width \(h\)-hop slice of \(v\) can be defined as \(N_h^b(v) = \bigcup_{u\in N_h(v)} N^b(u)\).
\end{definition}

% \begin{figure}[!htbp]
%     \centering
%     \includegraphics[width=0.75\linewidth]{figures/Fig1.pdf}
%     \caption{Illustration of the \(b\)-width \(h\)-hop slice of node. In this graph, \(b=0\), \(h=\left\{0,1,2\right\}\).}
%     \label{fig:nei012}
% \end{figure}

It is worth noting that when \(b = 0\) the \(b\)-width neighbor of each node \(v \in \mathcal{V}\) is the node itself, which is also equivalent to its \(b\)-hop neighbor, i.e., \(N^0(v) = N_0(v) = \{v\}\). In Figure~\ref{fig:b=0h=123}, we set the slice width \(b\) as 0 and vary the slice hop \(h\) from 0 to 2. Then we get the \(0\)-width \(0\)-hop slice, \(0\)-width \(1\)-hop slice and \(0\)-width \(2\)-hop slice of \(v_1\), i.e., \(N_0^0(v_1)\), \(N_1^0(v_1)\) and \(N_2^0(v_1)\), which capture the neighborhood structure information of \(v_1\) in different hierarchies. In Figure \ref{fig:b=1h=1}, we show that the 1-width 1-hop slice \(N_1^1(v_1)\) of node \(v_1\) is the union set of \(N^1(v_2), N^1(v_3)\), and \(N^1(v_4)\).

% Figure \ref{fig:nei012} shows the illustration of \(\left\{ N_0^0(v_1), N_1^0(v_1), N_2^0(v_1)\right\}\). Fixing the slice width \(b\) and varying the hop depth \(h\), we can gradually capture the structure around the node; Meanwhile, fixing the hop depth \(h\) and changing the slice width \(b\), we can capture the structure between neighbors of h-hop node neighbors. In summary, the \(b\)-width \(h\)-hop Slice of Node can comprehensively and flexibly catch the characteristics of a node's at most \((h + b)\) hop neighborhood of the given node.

\begin{figure}[!htbp]
    \hspace*{\fill}
    \centering
    \subfloat[\(b=0, h\in\{0,1,2\}\)]{
    \includegraphics[width=0.45\linewidth]{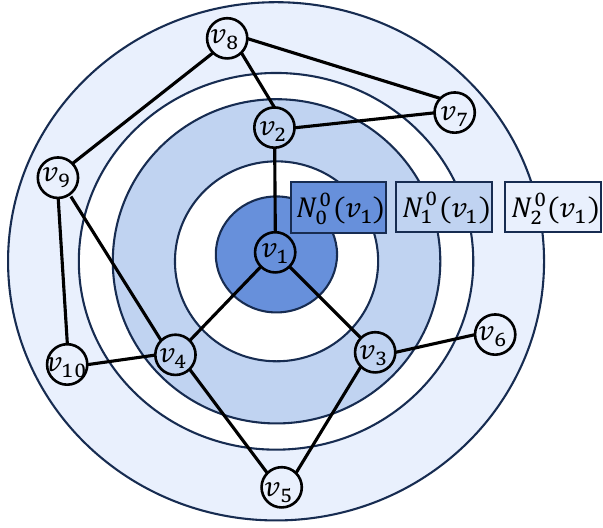}
    \label{fig:b=0h=123}
    }
    \hfill
    \hfill
    \subfloat[\(b=1, h=1\)]{
    \includegraphics[width=0.45\linewidth]{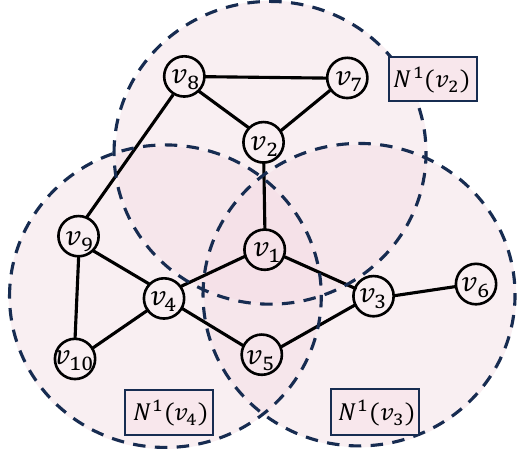}
    \label{fig:b=1h=1}
    }
    \hspace*{\fill}
    \caption{(a) shows \(b\)-width \(h\)-hop slice of node \(v_1\), where \(b\) is fixed as 0 and \(h\) ranges from 0 to 2. (b) shows 1-width 1-hop slice of node \(v_1\), where \(N_1^1(v_1)\) is the union set of \(N^1(v_2), N^1(v_3)\), and \(N^1(v_4)\) since \(v_2,v_3,v_4\in N_1(v_1)\).}
    \label{fig:encodingOfSlice}
\end{figure}

Then, we encode each slice into a unique sequence \(S_h^b(v)\). Given width \(b\) and hop \(h\), for each node \(v\) in the graph, we find its \(h\)-hop neighbors \(N_h(v)\) and sort them by their eigenvector centrality (EC) \cite{eigenvectorCentrality} values in descending order. For each node \(u\) in the sorted \(N_h(v)\), we concatenate the nodes in its \(i\)-hop neighbors \(N_i(u)\), sorted by EC in descending order to get \(S_i(u)\). Then, we concatenate \(S_i(u)\) from \(i=0\) to \(b\) to obtain \(S^b(u)\). The encoding of slice \(S^b_h(v)\) is the concatenation of \(S^b(u)\) for all the nodes \(u\) in \(N_h(v)\). Finally, each node in the encoding of the slice is assigned a label by applying \(l(\cdot)\).
For example, in Figure \ref{fig:b=0h=123}, \(S^0_1(v_1)=l([v_4, v_3, v_2])\); and in Figure \ref{fig:b=1h=1},  \(S^1_1(v_1) = l([v_4, v_9, v_1, v_{10}, v_5 ] \ \Vert \ [ v_2, v_1, v_8, v_7 ] \ \Vert \ [v_3, v_1, v_5, v_6])\), where \(\Vert\) concatenates different \(S^b(u)\). The pseudo-code is given in Algorithm \ref{alg:slice_encoding} in \ref{psucode:EncodingSlice}.

In $\mathcal{R}$-convolution graph kernels, the similarity value between each pair of slices is one (if they have the same encoding) or zero (if they have different encodings even the difference is just a little bit). This strict comparison does not consider the fact that there exists some similarity even in two slices that have different encodings. Several studies~\cite{perozzi2014deepwalk,yanardag2015deepGraphKernels} have shown that the substructures in graphs follow a power-law distribution, just like the word distribution in natural language. Thus, to circumvent the above difficulty, we generate a deep embedding for each slice to soften the strict similarity comparison and to compute slice similarity more accurately by employing Natural Language Models (NLM).

% To circumvent this difficulty, we generate a deep embedding for each Slice to soften the strict similarity comparison and to more precisely compute Slice similarity. Since each Slice is encoded as a sequence of its node labels and contains the neighborhood structural or semantic information of a node, we can use Natural Language Models (NLM) to learn its deep embedding.

More specifically, given a slice encoding, we consider each node label in the encoding as a word and the whole encoding as a sentence. Then an NLM can be trained to learn the embedding of each node label. We define the deep embedding of a slice as follows:

\begin{definition}[Deep Slice Embedding]
Given a label embedding function (learned by NLMs) \(g: \Sigma \to \mathbb{R}^d\), where \(d\) is the dimension of the embedding. The embedding of a \(b\)-width \(h\)-hop slice of node \(v\) is defined as follows:
\begin{equation}
    \mathbf{x}_h^b(v) = \alpha \cdot \mathbf{x}_{h-1}^b(v) + \sum_{u \in S_h^b(v)}g(u)
    \label{eq:node_emb}
\end{equation}
where \(\mathbf{x}_0^b(v):= g(l(v))\), \(\alpha\) is a decay coefficient of the \((h-1)\)-th hop's node label embedding and \(S_h^b(v)\) is the encoding of the \(b\)-width \(h\)-hop slice of node \(v\).
\label{df:slice-emb}
\end{definition}

\subsection{Deep Hierarchical Graph Alignment Kernels}

% \textcolor{blue}{In this section, we describe how to derive the Hierarchical Uniform Graph Clustering Kernel (HUGE) (the pseudo-code is in Algorithm \ref{alg:huge}). First we define two basic kernels, namely the Clustering Kernel and the Uniform Clustering Kernel.}

To consider the topological position information of slices, we propose to align slices from different graphs by clustering in their deep embedding space. Slices within the same cluster are aligned and have the same feature map. This not only addresses the neglect of the topological position information of substructures in conventional $\mathcal{R}$-convolution graph kernels but also reduces the computational overhead since comparison will only be carried out on slices from different clusters.
% In this section, we describe how to derive the Deep Hierarchical Graph Alignment Kernels. 
We first define the Deep Alignment Kernel as follows:

% \begin{definition}[Isolation Kernel (IK)]
% \label{def:isolationKernel}
%     Let \(\chi=\{x_1,\ldots,x_n\}\) be a arbitrary dataset sampled from an unknown distribution \(\mathbb{D}\), we assume there is a space partitioning mechanism that produce \(p\) partitions \(P\) on \(D\). Each partition \(\theta \in P\) isolates data point \(z\in\theta\) from the rest of points in other partitions. The Isolation Kernel (IK) on \(x,y\in \chi\) is defined as follows:
% \begin{equation}
%     k_{p}(x, y) = \mathbb{E}_{\mathcal{P}_{p}(D)}
%     [\mathbb{I}(x,y\in \theta | \theta \in P)]
% \end{equation}
% Here \(\mathcal{P}_{p}(D)\) denotes the set of all \(p\) partitions that are admissible on D. \(\mathbb{I}(B)\) is the indicator function which \(\mathbb{I}(B)=1\) if \(B\) is true, otherwise, \(\mathbb{I}(B)=0\).
% \end{definition}

% In Isolation Kernel, the similarity between two data points is defined by the expectation of falling into into the same partition taken over the distribution on \(\mathcal{P}_{p}(D)\). Since the partitioning mechanism in IK is chosen arbitrarily, we give a variant of IK based on clustering method as follows:

\begin{definition}[Deep Alignment Kernel (DAK)]
\label{df:dak}
    Let \(\Psi \subset \mathfrak{M}\) be an arbitrary set of clustering methods chosen from all admissible clustering methods \(\mathfrak{M}\). Let \(U(\Psi)\) denote a uniform distribution on \(\Psi\). For any clustering method \(\psi: \mathbf{X}_h^b \to \mathcal{C_\psi}\) in \(\Psi\), it separates each point (b-width h-hop slice embedding) \(\mathbf{x} \in \mathbf{X}_h^b\) into one of \(|\mathcal{C_\psi}|\) clusters, i.e., \(\mathcal{C_\psi} = \{\mathcal{C}_{\psi,1}, \mathcal{C}_{\psi,2}, \ldots, \mathcal{C}_{\psi, |\mathcal{C_\psi}|}\}\). The Deep Alignment Kernel quantifies the likelihood of assigning two points (slice embeddings), \(\mathbf{x}, \mathbf{y} \in \mathbf{X}_h^b \), to the same cluster across all clustering methods in the set \(\Psi\), which is defined as follows:
       \begin{equation}
        \kappa_h^b(\mathbf{x}, \mathbf{y}) = \mathbb{E}_{\psi \sim U(\Psi)} [\mathbb{I}(\mathbf{x}, \mathbf{y} \in \mathcal{C}| \mathcal{C}\in \mathcal{C}_\psi)] 
        % &= \frac{1}{| \Psi |} \sum_{\psi \in \Psi} 
        % \mathbb{I}(x, y \in \mathcal{C}| \mathcal{C}\in C_\psi) \\
        % &\simeq \frac{1}{T | \Psi |} \sum_{\psi \in \Psi} \sum_{t=1}^{T} 
        % \mathbb{I}(x, y\in \mathcal{C}^{(t)}| \mathcal{C}^{(t)}\in \mathbf{C}^{(t)}_\psi) 
    \label{eq:uck1}
    \end{equation}
\end{definition}

Given that many clustering methods (e.g., K-means \cite{k-means}) heavily rely on the initialization points, we perform each clustering method \(T\) times to approximate the expectation. The approximation is defined as:
\begin{equation}
     \label{eq:gak2}
         \kappa_h^b(\mathbf{x}, \mathbf{y}) \simeq \frac{1}{T | \Psi | } 
        \sum_{\psi \in \Psi} \sum_{t=1}^{T} 
        \sum_{i=1}^{|\mathcal{C}_\psi |} \mathbb{I}(\mathbf{x}\in \mathcal{C}^{(t)}_{\psi,i}) \mathbb{I}(\mathbf{y}\in \mathcal{C}^{(t)}_{\psi,i})
\end{equation}
where \(\mathbb{I}(a)\) is the indicator function that \(\mathbb{I}(a)=1\) if \(a\) is true, otherwise \(\mathbb{I}(a)=0\). \(\mathcal{C}^{(t)}_{\psi,i}\) represents the \(i\)-th cluster at the \(t\)-th experiment under clustering method \(\psi\).

\begin{theorem}
    DAK is positive semi-definite.
\label{th:semi-pos-DAK}
\end{theorem}

With Theorem \ref{th:semi-pos-DAK} (see \ref{proof:DAK_pos} for the proof), we can define explicit feature maps in the Reproducing Kernel Hilbert Space (RKHS) \(\mathcal{H}\) for the DAK as follows:

% The CK and UCK are proved as positive semi-definite in Appendix \ref{proof:ck_pos}. Thus, it define a RKHS \(\mathcal{H}\). Unlike common kernel such as Gaussian Radial Basis Function (RBF) or Polynomial Kernel, the CK and UCK have very explicit feature maps in \(\mathcal{H}\). The feature map of the UCK is shown as below:

% \begin{definition}[Feature map of the CK and UCK]
%     Let \(\mathcal{CI}_\psi: \chi \to \{0, 1\}^{| \mathbf{C_\psi} |}\) be a cluster indicator of clustering method \(\psi\), where for any \(x \in \chi \), \(CI(x)_i = 1\) if and only if \(x \in \mathcal{C}_i\). The feature map \(\phi_\Psi(x)\) of the CK under \(\Psi\) is defined as:
%     \begin{equation}
%     \phi_\Psi(x) = \mathbb{E}_{\psi \sim P(\Psi)}\mathcal{CI}_\psi(x)
%     \end{equation}
%     And let \(\mathcal{CI}^{(t)}_\psi: \chi \to \{0, 1\}^{| \mathbf{C^{(t)}_\psi} |}\) be a cluster indicator of clustering method \(\psi\) at experiment t. The feature map \(\phi_{U(\Psi)}\) of the UCK after T experiments is defined as:
%     \begin{equation}
%     \begin{split}
%         \phi_{U(\Psi)}(x) &= \frac{1}{\sqrt{T| \Psi |}} Concat( \\ &Concat(\mathcal{CI}^{(t)}_\psi(x)| t=1, 2, \ldots, T)| \psi \in \Psi)
%     \label{eq:map-uck}
%     \end{split}
%     \end{equation}
%     Here, \(\phi_{U(\Psi)}(x) \in \{0,1\}^{ (\sum_{\psi \in \Psi}| \mathbf{C}_\psi | ) \times T}\) represents the concatenation of all \(\mathcal{CI}\) for T experiments of \(\psi\)'s cluster set \(\mathbf{C}^{(t)}_\psi\). 
% \end{definition}

\begin{definition}[Feature Map of DAK]
\label{df:featureMapDAK}
    Given the embeddings \(\mathbf{X}_h^b\) of b-width h-hop slices, let \(\mathbbm{1}^{(t)}_\psi: \mathbf{X}_h^b \to \{0, 1\}^{| \mathcal{C}^{(t)}_\psi |}\) be a cluster indicator of clustering method \(\psi\) at the \(t\)-th experiment. For any \(\mathbf{x} \in \mathbf{X}_h^b \), \(\mathbbm{1}^{(t)}_\psi(\mathbf{x})_i = 1\) if and only if \(\mathbf{x} \in \mathcal{C}^{(t)}_{\psi,i} \in \mathcal{C}^{(t)}_\psi\), otherwise \(\mathbbm{1}^{(t)}_\psi(\mathbf{x})_i = 0\). Then, the feature map \(\phi_h^b\) of the DAK after $T$ experiments is defined as:
    \begin{equation}
    % \begin{split}
        \phi_h^b(\mathbf{x}) = \frac{1}{\sqrt{T| \Psi |}}\parallel_{\psi \in \Psi}\parallel_{t=1}^T\mathbbm{1}^{(t)}_\psi(\mathbf{x}) 
        % Concat( \\ &Concat(\mathcal{CI}^{(t)}_\psi(x)| t=1, 2, \ldots, T)| \psi \in \Psi)
    \label{eq:map-uck}
    % \end{split}
    \end{equation}
    where $\parallel$ means concatenation, \(\phi_h^b(\mathbf{x}) \in \mathbb{R}^{ T (\sum_{\psi \in \Psi}| \mathcal{C}_\psi | )}\) represents the concatenation of all \(\mathbbm{1}^{(t)}_\psi(\mathbf{x})\) for $T$ experiments of \(\psi\) in the clustering method set $\Psi$. 
\end{definition}

Therefore, Equation \ref{eq:gak2} can be rewritten as:
% \begin{equation}
%     h_{\Psi}(x, y) = \langle \phi_\Psi(x), \phi_\Psi(y) \rangle 
%     \label{eq:re-ck}
% \end{equation} 
\begin{equation}
    \kappa_h^b (\mathbf{x},\mathbf{y})=\left\langle  \phi_h^b(\mathbf{x}), \phi_h^b(\mathbf{y}) \right\rangle 
    \label{eq:re-uck}
\end{equation}

Note that Equation~\ref{eq:re-uck} can directly compute the similarity between two slices extracted around two nodes residing in the same or different graphs. We extract the $b$-width $h$-hop slices around all the nodes in each graph. For the feature maps of DAK in each graph, we treat them as i.i.d. samples from an unknown distribution.  Next, we utilize kernel mean embedding \cite{kernel-mean-embedding} to construct the Deep Graph Alignment Kernel (DGAK) between two graphs $G_1$ and $G_2$ from the perspective of distribution as follows:
\begin{equation}
\begin{split}
\mathcal{K}_h^b(G_1,G_2)&=\langle \frac{1}{|\mathbf{X}_h^b|}\sum_{\mathbf{x}\in\mathbf{X}_h^b}\phi_h^b(\mathbf{x}), \frac{1}{|\mathbf{Y}_h^b|}\sum_{\mathbf{y}\in\mathbf{Y}_h^b}\phi_h^b(\mathbf{y}) \rangle\\
&=\frac{1}{|\mathbf{X}_h^b||\mathbf{Y}_h^b|}\sum_{\mathbf{x}\in\mathbf{X}_h^b}\sum_{\mathbf{y}\in\mathbf{Y}_h^b}\kappa_h^b (\mathbf{x},\mathbf{y})
\end{split}
\label{eq:DGAK}
\end{equation}
where $\mathbf{X}_h^b$ and $\mathbf{Y}_h^b$ are the deep embeddings of $b$-width $h$-hop slices in graphs $G_1$ and $G_2$, respectively.

Finally, our Deep Hierarchical Graph Alignment Kernel (DHGAK) is defined as follows:
\begin{equation}
\mathcal{K}(G_1,G_2)=\sum_{h=1}^H\mathcal{K}_h^b(G_1,G_2)
\label{eq:DHGAK}
\end{equation}

Since the sum of positive semi-definite kernels is also positive semi-definite, both DGAK and DHGAK are positive semi-definite.
% \begin{theorem}
%     DHGAK is positive semi-definite.
% \label{th:semi-pos-DHGAK}
% \end{theorem}
% See \ref{proof:semi-pos-DHGAK} for the proof.

% The pseudo-code of HUGE is given in Algorithm \ref{alg:huge}.

% Finally, Equation \ref{eq:huge} can be rewritten as:
% \begin{equation}
%     \mathcal{K}_{U(\Psi)}(G_1, G_2; k) = \langle \hat{\Phi}_{U(\Psi)}(G_1;k),
%     \hat{\Phi}_{U(\Psi)}(G_2;k) \rangle
% \label{eq:huge-ker}
% \end{equation}

% \textcolor{blue}{Finally, We use \(\mathcal{K}_U(G_1, G_2)\) and \(\hat{\Phi}_U(G)\) to denote the \(\mathcal{K}_{U(\Psi)}(G_1,G_2; k)\) and \(\hat{\Phi}_{U(\Psi)}(G;k)\) for short, respectively. The pseudo-code of HUGE is given in Algorithm \ref{alg:huge}. The final formulation of HUGE is as follows:}

% \textcolor{blue}{
% \begin{equation}
% \begin{split}
%         &\mathcal{K}_{U}(G_1, G_2)=\langle \hat{\Phi}_{U}(G_1),\hat{\Phi}_{U}(G_2) \rangle
% \end{split}
% \label{eq:finalhuge}
% \end{equation}
% }

% Thereafter, \(\mathcal{K}_U(G_1, G_2)\) and \(\hat{\Phi}_U(G)\) are shorthand for \(\mathcal{K}_{U(\Psi)}(G_1,G_2; k)\) and \(\hat{\Phi}_{U(\Psi)}(G;k)\) under the given $\Psi$ and hop $k$, respectively. The pseudo-code of HUGE is given in Algorithm\ref{alg:huge}. 

\subsection{Theoretical Analysis of {\namemodel}}
\label{sec:theoreticAnalysis}
% In this section, we elaborate on the feasibility and algorithm complexity of {\namemodel}. 
We explicit the transitivity of alignment operation in DAK, DGAK, and DHGAK at first. Given the set of clustering method \(\Psi\), experiment times \(T\in\mathbb{N}\) and \(\mathbf{x}, \mathbf{y} \in \mathbf{X}_h^b\). Let the \(\mathcal{M}_{h, \psi}^{b, t}: \mathbf{X}_h^b \times \mathbf{X}_h^b \to \{0,1\}\) be the alignment operation at the \(t\)-th experiment under \(\psi\in\Psi\) , which satisfies:
\begin{equation}
    \mathcal{M}_{h, \psi}^{b, t}(\mathbf{x}, \mathbf{y}) = 
    \begin{cases}
        1 & \text{ if } \mathbf{x} \text{ and } \mathbf{y}  \text{ are assigned to the same} \\  & \text{ cluster at the } t \text{-th experiment under }\psi \text{;} \\
        0 & \text{otherwise.}
    \end{cases}
\end{equation}

Therefore, DAK can be re-expressed as:
\begin{equation}
    \kappa_h^b(\mathbf{x}, \mathbf{y})= \frac{1}{T|\Psi|}
    \sum_{\psi\in\Psi}\sum_{t=1}^T \mathcal{M}_{h, \psi}^{b, t}(\mathbf{x}, \mathbf{y})
\end{equation}

DGAK and DHGAK can also be re-expressed in the sum of alignment operations. The theorem (See \ref{proof:transitive} for the proof) below shows the alignment operation is transitive.
    \begin{theorem}
        The alignment operation \(\mathcal{M}_{h,\psi}^{b,t}\) in DAK, DGAK and DHGAK is transitive.
        \label{th:transitive}
        \end{theorem}

Secondly, we care about the feasibility of using kernel mean embedding in Equation \ref{eq:DGAK}. We treat the feature maps of DAK in each graph as i.i.d. samples from an unknown distribution \(\mathbb{G}\) and take kernel mean embedding of them to construct DGAK and DHGAK. When considering distribution classification tasks, Distributional Risk Minimization (DRM) \cite{muandet2012learningfromdit} assures that any optimal solutions of any loss function can be expressed as a finite linear combination of the mean embedding of \(\mathbb{G}_1, \ldots, \mathbb{G}_N\). DRM guarantees the feasibility of graph kernels derived by kernel mean embedding for graph classification problems.

Finally, 
% the validation of DHGAK heavily relies on the clustering methods we chose. 
Theorem \ref{th:linearSeparationPropertyOfDHGAK} is introduced to assure that there exist some sets of clustering methods \(\Psi_o \subset \mathfrak{M}\), derived from which the feature map of {\namemodel} is linearly separable in its RKHS for any dataset.  

\begin{theorem}[Linear Separation Property of {\namemodel}]
    Given a graph dataset \(\mathcal{D}=\{(G_i, y_i)\}_{i=1}^{N}\), let \(\Psi_o\) denote the set of clustering method \(\{\psi_{j}\}_{j=1}^{|\Psi_o|}\), where \(\psi_j: \mathbf{X}_h^b \to \mathcal{C}_{\psi_{j}}\) separates each point \(\mathbf{x}_h^b\) (the embedding of a slice extracted around a node $v$) to the cluster \(\mathcal{C}_{\psi_j, k}\) if and only if \(v \in \mathcal{V}_k\) belonging to \(G_k\). For any \(H,T\in \mathbb{N}\), The feature map of DHGAK is linearly separable for one class graphs \(\{(G_p, y_p)\}\) from the other class graphs \(\{(G_q, y_q): y_q \ne y_p\}\).
% given graph label y,  the feature map of DHGAK is linearly separable for graphs \(\{G_p: y_p=y\}\) from the other graphs \(\{G_q : y_q \ne y\}\).
    \label{th:linearSeparationPropertyOfDHGAK}
\end{theorem}

The proof of Theorem \ref{th:linearSeparationPropertyOfDHGAK} is given in \ref{proof:LinearSeparationPropertyOfDHGAK}. It's worth noticing that \(\Psi_o\) is only utilized for theoretical derivations. In addition, the theorem offers a rough estimator for the number of clusters, denoted as $|\mathcal{C_\psi}| \simeq |\mathcal{D}|$. This estimation is useful for the detailed implementations of {\namemodel}.
% and it is used in our design of experiments \ref{sec:experimentalSetup}.

The worst-case time complexity of {\namemodel} is bounded by \(\mathcal{O}(HN(n^2 + ne) \times \mathcal{T}_{NLM} + HT|\Psi| \times \mathcal{T}_C(Nn))\), where \(\mathcal{T}_{NLM}\) is the inference time complexity of the NLM per token, \(n, e\) are the average number of nodes and edges in graphs, and \(\mathcal{T}_C(Nn)\) is the average time complexity of the clustering methods we select. Refer to 
Appendix \ref{sec:TimeComplexityAnalysis} for a more detailed analysis of the time complexity of {\namemodel}.

% \mathcal{O}(HN(n^3e + n^2e^2)\mathcal{T}_{NLM} + HT|\Psi|\mathcal{T}_C(Nn))
% \(\mathcal{O}(L+HTNn\times f(Nn)+Nn)\)

% In short, DHGAK we propose is alignment transitive, which guarantees its positive semi-definiteness. Kernel mean embedding is shown to be suitable for distributional classification, i.e., graph classification tasks. And theoretically, for any dataset, there exists a family of clustering methods derived from which DHGAK's feature maps are linearly separable in the RKHS space. All this ensure the effectiveness of our proposed method on graph classification tasks.

In short, DHGAK is proven to be alignment transitive, and theoretically, for any dataset, there exists a family of clustering methods derived from which DHGAK's feature maps are linearly separable in the RKHS.

\section{Experimental Evaluation}

In this section, we analyze the performance of {\namemodel} compared to some state-of-the-art graph kernels. First, we give the experimental setup and introduce two realizations of {\namemodel}. Next, we demonstrate the performance of {\namemodel} on graph classification, the parameter sensitivity analysis of {\namemodel}, and the ablation study to reveal the contribution of each component in {\namemodel}. Finally, we compare the running time of each graph kernel.
% Experimental results show that the proposed method achieves state-of-the-art performance on most graph benchmark datasets.

\subsection{Experimental Setup}
\label{sec:experimentalSetup}

The generalization of {\namemodel} allows one to use any combination of Natural Language Models (NLM) for deep embedding and clustering methods for alignment to implement {\namemodel}. The NLMs we select are BERT \cite{BERT} and word2vec \cite{word2vec}, and the clustering method we use is K-means. By doing so, we get two realizations of {\namemodel}: DHGAK-BERT and DHGAK-w2v. See \ref{sec:detailsOfNLM} and \ref{sec:additionReals} for more details of the NLMs and realizations.
% A HUGE implementation of BERT \cite{BERT} as NLM and k-means \cite{k-means} as clustering method was used for experimental comparison. Since unlike conventional models that process language input exclusively from left to right or vice versa, BERT considers both side context during its training process. It's used for learn embedding of Slice by fine-tuning. Considering the discussion of risk deviation bound and time complexity in section \ref{sec:theoreticAnalysis}, we use the k-means clustering method and clustering with a low complexity. See for detail of fine-tuning BERT. 

We evaluate {\namemodel} on 16 real-world datasets downloaded from \cite{KKMMN2016Benchmark}. For datasets without node labels, we use the degree of each node as its node label. More details about datasets are provided in \ref{sec:detailsOfDatasets}. All experiments were conducted on a server equipped with a dual-core Intel(R) Xeon(R) Gold 6226R CPU @ 2.90GHz, 256 GB memory, and Ubuntu 18.04.6 LTS with 6 RTX 3090 GPUs. 
% The implementation of {\namemodel} is written in Python and is available at Github \footnote{\url{code}}.

\begin{table*}[!htbp]
    \centering
    \resizebox{\textwidth}{!}{%
    \begin{tabular}{l|c|c|c|c|c|c|c|c|c|c|c}
        \toprule
        Dataset&  {\namemodel}-BERT &{\namemodel}-w2v&  PM&  WL&  WWL&  FWL&  RetGK&  WL-OA&  GWL& GAWL & GraphQNTK\\
        \midrule
        KKI & $\mathbf{58.9\pm8.2}$ &$56.7\pm5.2$
& $52.3\pm2.5$& $50.4\pm2.8$& $55.4 \pm 5.1$& $54.3\pm19.8$& $48.5\pm3.0$& $54.4 \pm12.8$& $\underline{57.8 \pm4.0^\dag}$& $56.8\pm23.1$&$56.8\pm22.9$\\
        MUTAG 
        &  $\mathbf{90.9\pm4.2}$ &$\underline{90.4\pm4.6}$
&  $86.7 \pm0.6^\dag$&  $82.1\pm0.4^\dag$&  $87.3\pm1.5^\dag$&  $85.7\pm7.5$&  $90.3\pm1.1^\dag$&  $84.5\pm1.7^\dag$&  $86.0\pm1.2$&  $87.3\pm 6.3^\dag$&  $88.4\pm6.5^\dag$\\
        PTC\_MM 
        &  $\underline{70.5\pm4.6}$&$\mathbf{70.9\pm5.6}$&  $63.4 \pm 4.4$&  $67.2\pm1.6$&  $69.1\pm5.0$&  $67.0\pm5.1 $&  $67.9\pm1.4^\dag$&  $65.2 \pm6.4$&  $65.1\pm1.8$&  $66.3\pm5.3 $& $66.4\pm12.7$\\
        PTC\_MR 
        &  $\mathbf{66.6\pm7.7}$ &$63.7\pm7.3$&  $60.2 \pm0.9^\dag$&  $61.3\pm0.9$&  $\underline{66.3\pm1.2^\dag}$&  $59.3\pm7.3$&  $62.5\pm1.6^\dag$&  $63.6\pm1.5^\dag$&  $59.9\pm1.6$&  $59.0\pm4.4$& $59.4\pm9.9$\\
        PTC\_FM 
        &  $\underline{66.2\pm5.2}$&$\mathbf{66.5\pm6.0}$&  $61.0 \pm3.9$&  $64.4\pm2.1$&  $65.3\pm6.2$&  $61.0\pm6.9$&  $63.9\pm1.3^\dag$&  $62.5 \pm6.3$&  $62.6 \pm1.9^\dag$&  $63.9\pm4.4$& $63.8\pm5.3$\\
        PTC\_FR 
        & $\mathbf{72.1\pm5.8}$ &$\underline{71.3\pm4.9}$& $67.8 \pm 2.2$& $66.2\pm1.0$& $67.3\pm4.2$& $67.2\pm4.6$& $67.8\pm1.1^\dag$& $68.4 \pm4.3$& $66.0\pm1.1$& $64.7\pm3.6$&$68.9\pm9.0$\\
        BZR 
        &  $\mathbf{89.4\pm3.3}$ &$88.4\pm4.5$&  $81.5 \pm 2.0$&  $87.
        3\pm 	 0.8$&  $80.5\pm1.7$&  $87.6\pm5.0$&  $86.4\pm1.2^\dag$&  $81.5 \pm 4.7$&  $87.7\pm0.4$&  $\underline{89.1\pm4.1}$& $83.1\pm4.7$\\
        COX2 
        &  $\mathbf{84.6\pm2.2}$ &$83.1\pm4.0$&  $78.2 \pm 0.8$&  $81.2\pm 	 1.1$&  $83.1 \pm3.0$&  $79.2\pm2.9$&  $80.1\pm0.9^\dag$&  $78.4 \pm 1.1$&  $82.9\pm1.1$&  $\underline{83.5\pm3.1}$& $79.5 \pm1.6$\\
        DHFR 
        &  $\mathbf{85.5\pm4.4}$ &$\underline{84.0\pm4.6}$&  $76.9 \pm 3.9$&  $82.4 	 \pm0.9$&  $82.7 \pm 3.5$&  $82.8\pm5.8$&  $81.5\pm0.9^\dag$&  $82.7 \pm4.5$&  $80.6\pm0.8$&  $83.7\pm3.0$& $72.1\pm4.8$\\
        ENZYMES 
        & $\mathbf{65.3\pm7.9}$ &$57.3\pm5.4$& $40.3 \pm0.3
        ^\dag$& $52.2\pm 1.3^\dag$& $59.1\pm0.8^\dag$& $51.8\pm5.5 $& $\underline{60.4\pm0.8^\dag}$& $59.9\pm1.1^\dag$& $54.5\pm1.6 $& $58.5\pm4.8 $&$35.8\pm5.3$\\
        PROTEINS 
        & $\mathbf{76.6\pm4.3}$&$75.2\pm4.6$& $74.8 \pm3.9$& $75.5\pm0.3$& $74.3\pm0.6^\dag$& $74.6\pm3.8$& $75.8\pm0.6^\dag$& $\underline{76.4\pm 0.4 ^\dag}$& $73.7\pm0.5 $& $74.7\pm3.0^\dag$&$71.1 \pm3.2^\dag$\\
        DD 
        & $\underline{81.0\pm4.7}$&$80.1\pm3.1$& $ 77.8 \pm0.5^\dag$& $79.8\pm0.4^\dag$& $79.7\pm0.5^\dag$& $78.4\pm2.4$& $\mathbf{81.6\pm0.3^\dag}$& $79.2\pm0.4^\dag$& $79.0\pm0.8^\dag$& $78.7\pm2.8^\dag$&$79.6\pm4.7$\\
        NCI1 &  $\underline{85.9\pm2.0}$&$85.2\pm2.1$&  $ 72.9\pm0.5^\dag$&  $82.2 \pm 0.2^\dag$&  $85.8\pm0.3^\dag$&  $85.4\pm1.6$&  $84.5\pm0.2^\dag$&  $\mathbf{86.1\pm0.2^\dag}$&  $85.3\pm0.4^\dag$&   $\underline{85.9 \pm1.2^\dag}$& $77.2 \pm2.7^\dag$\\
        IMDB-B& $\mathbf{75.3\pm2.7}$ &$73.5\pm1.9$& $74.8 \pm 2.3$& $73.8\pm3.9 $& $73.3\pm4.1 $& $72.0\pm4.6$& $71.9\pm1.0^\dag$& $\underline{75.2\pm2.8}$& $73.7\pm1.3^\dag$& $74.5\pm4.1^\dag$&$73.3 \pm3.6^\dag$\\
        IMDB-M& $\mathbf{52.1\pm2.4}$ &$50.
        0\pm2.3$& $50.7 \pm 2.8$& $51.1\pm2.7$& $50.4\pm3.4$& $50.0\pm2.9$& $47.7\pm0.3^\dag$& $50.9 \pm 2.8$& $50.6\pm0.5$& $\underline{51.7\pm 5.2^\dag}$& $48.1 \pm 4.3^\dag$\\
        COLLAB& $\mathbf{81.7\pm1.5}$ &$\underline{81.6\pm1.7}$& $80.5 \pm1.5$& $78.7\pm1.5$& $79.1\pm1.5$& $78.1\pm1.2$& $81.0\pm0.3^\dag$& $80.7\pm0.1^\dag$& $80.6\pm0.3$& $81.5\pm2.0^\dag$&$>24h$\\
        \bottomrule
    \end{tabular}
    }
    \caption{10-fold cross-validation \(C\)-SVM classification accuracy (\%). \(\dag\) means the result is adopted from the original papers (the highest accuracy marked in \textbf{bold} and the second highest accuracy marked as \underline{underlined}).}
    \label{table:acc}
\end{table*}

% Datasets IMDB is a movie collaboration dataset, which node represent actors/actresses and edges represent that the two actors/actresses appear in the same movie. For each actor/actress, a corresponding collaboration graph (ego network) is derived.  Dataset REDDIT is online threads graphs, which nodes represent users and edges represent the responses between users. 

% The parameters in HUGE are chosen as follows. For small datasets(size smaller than 1000), we fine-turn the pretrained BERT model \cite{BERT-base-uncased}  for 0 (no fine-tuned) or 3 epochs to get the embedding of Section \ref{sec:node_emb_bert}. The decay factor \(\alpha\) in Definition \ref{df:slice-emb} takes 0.6, the width of slice \(b\) is chosen from \(\{0,1,2\}\), maximum hop \(k_{max}\) is chosen from \(\{1,3,5,\ldots,9\}\). As discussed in Section \ref{sec:LinearSeparationPropertyinHUGE}, the set of clustering method is taken k-means \cite{k-means} for distribution rounded and we set the number of clustering of different \(k\) layer be same, around the size of graph dataset. That is \(|\mathbf{C}^{(k)}_\psi| = |D|*clusters\_factor, k=1, \ldots, k_{max}\), the \(clusters\_factor\) is chosen from 10 log-scale range from 0.1 to 5, i.e., \(\{0.1, 0.1*r, 0.1*r^2, \ldots,0.1*r^8, 5\}, r = (log_{10}(5)/log_{10}(-1))^{1/9}\) . And the times of experiment \(T\) is taken 3. For large datasets (size bigger than 1000)  as shown in Table \ref{tab:paramSearch}. 

The parameters of {\namemodel} are set as follows. As discussed in Theorem \ref{th:linearSeparationPropertyOfDHGAK}, we set the number of clusters for different hops \(h\) as the same, around the size of the graph dataset, i.e., \(|\mathcal{C}_\psi| = |\mathcal{D}|\times clusters\_factor\), where the \(clusters\_factor\) are searched as base-10 log scale. We fix the experiment times \(T\) to 3 and other parameters are grid-searched as shown in Table \ref{tab:paramSearch} via 10-fold cross-validation on the training data.

We compare {\namemodel} with 9 state-of-the-art graph kernels, i.e., PM \cite{PM}, WL \cite{WL-subtree_GK}, WWL \cite{WWL}, FWL \cite{GraphFiltrationKernels(FWL)}, RetGK \cite{RetGK}, WL-OA \cite{WL-OA}, GWL \cite{GWL}, GAWL \cite{GAWL}, and GraphQNTK \cite{GraphQNTK}. The parameters of these methods are set according to their original papers and implementations. We use the implementation of PM, WL, and WL-OA from GraKeL library \cite{graKeL}. All kernel matrices are normalized as \(K(i,j) = K(i,j)/\sqrt{K(i,i)K(j,j)}\). We use 10-fold cross-validation with a binary \(C\)-SVM \cite{libsvm} to test the classification performance of each graph kernel. The parameter \(C\) for each fold is independently tuned from \(\{10^{-3}, 10^{-2}, \ldots, 10^4\}\) using the training data from that fold. 
% and we also tune the parameter \(\lambda\) from \(\{10^{-4}, 10^{-3}, \ldots, 10\}\) for WWL and FWL. 
We report the average classification accuracy and standard deviation over the 10-fold. 

\begin{table}[!htbp]
    \small
    \centering
    \begin{tabular}{cc}
        \toprule
         parameter& search ranges\\
         \midrule
        (BERT) fine-tune epochs& $\{0,3\}$\\
         $b$& $\{0,1,2\}$ \\
         $H$& $\{1,3,\ldots,9\}$ \\
 $\alpha$&$\{0, 0.2, \ldots, 1\}$\\
         \(clusters\_factor\)& $\{0.1, 0.1\times r, \ldots,0.1\times r^8, 2\}$\\
         \bottomrule
    \end{tabular}

    \caption{Parameter grid ranges for datasets. \(r= 10^{(\frac{\text{log}_{10}2}{0.1})}\) represents the interval of base-10 log scale ranges. (BERT) represents the parameter only available for DHGAK-BERT.}
    \label{tab:paramSearch}
\end{table}

\subsection{Classification Results}

\subsubsection{Classification Accuracy}

The classification accuracy results of two realizations of DHGAK ({\namemodel}-BERT and {\namemodel}-w2v) are shown in Table \ref{table:acc}. 
% Some comparison results are taken from \cite{Tree++}. 
On 16 datasets, both {\namemodel}-BERT and {\namemodel}-w2v show superior performance compared to all the competitors on 14 datasets except DD and NCI1. For small datasets (datasets above ENZYMES in Table \ref{table:acc}), DHGAK-BERT and DHGAK-w2v outperform all the baselines on all the datasets. For large datasets (datasets below PROTEINS in Table \ref{table:acc}), DHGAK-w2v shows a slight drop in classification accuracy and is comparable to GAWL; while DHGAK-BERT outperforms all baselines on all datasets except DD and NCI1 and has the second best performance on DD and NCI1. More specifically, DHGAK significantly outperforms some baselines on some datasets. For example, on DHFR, PTC\_MM, PTC\_FR, and ENZYMES, DHGAK exhibits absolute improvements of 1.8\%, 1.8\%, 3.2\%, and 4.9\%, respectively, over other methods.
% We observe that on some small datasets (e.g., PTC\_MR, ENZYMES), the standard deviations of DHGAK's method are larger than those of the baselines. This can be attributed to the deep embedding method BERT that overfits small graphs.
% stochastic nature of the clustering methods. 
Overall, DHGAK achieves the best performance on 14 of 16 datasets and the second best performance on DD and NCI1.

\subsubsection{Parameter Sensitivity}

In this section, we study the parameter sensitivity of DHGAK. The parameters are: the number of clusters \(|\mathcal{C_\psi}|=|\mathcal{D}|\times cluster\_factor\), the maximum hop \(H\), the slice width \(b\), decay weight \(\alpha \in [0, 1]\) and experiment times \(T\) (See \ref{sec:Tsensitivity}). We run {\namemodel}-BERT with 3 fine-tuned epochs on four datasets DHFR, PTC\_FR, KKI, and IMDB-B from different categories. 

\begin{figure}[!htbp]
    \hspace*{\fill}
    \centering
    \subfloat[KKI]{\includegraphics[width=.45\columnwidth]
    {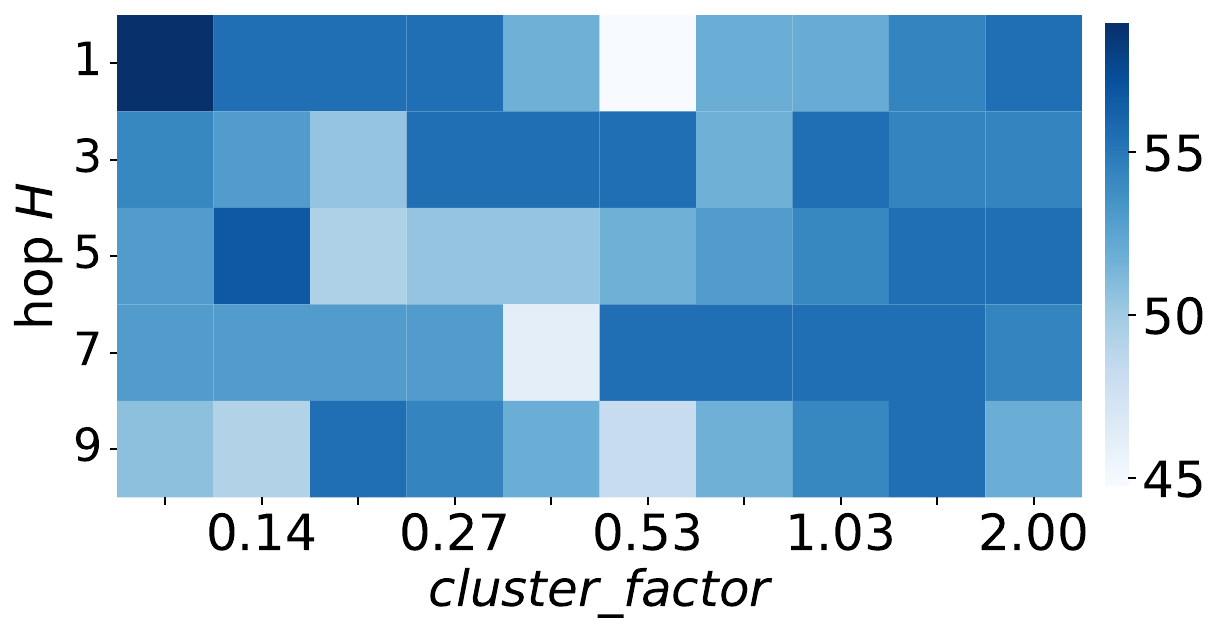}}
    \hfill
    \hfill
    \centering
    \subfloat[PTC\_FR]{\includegraphics[width=.45\columnwidth]{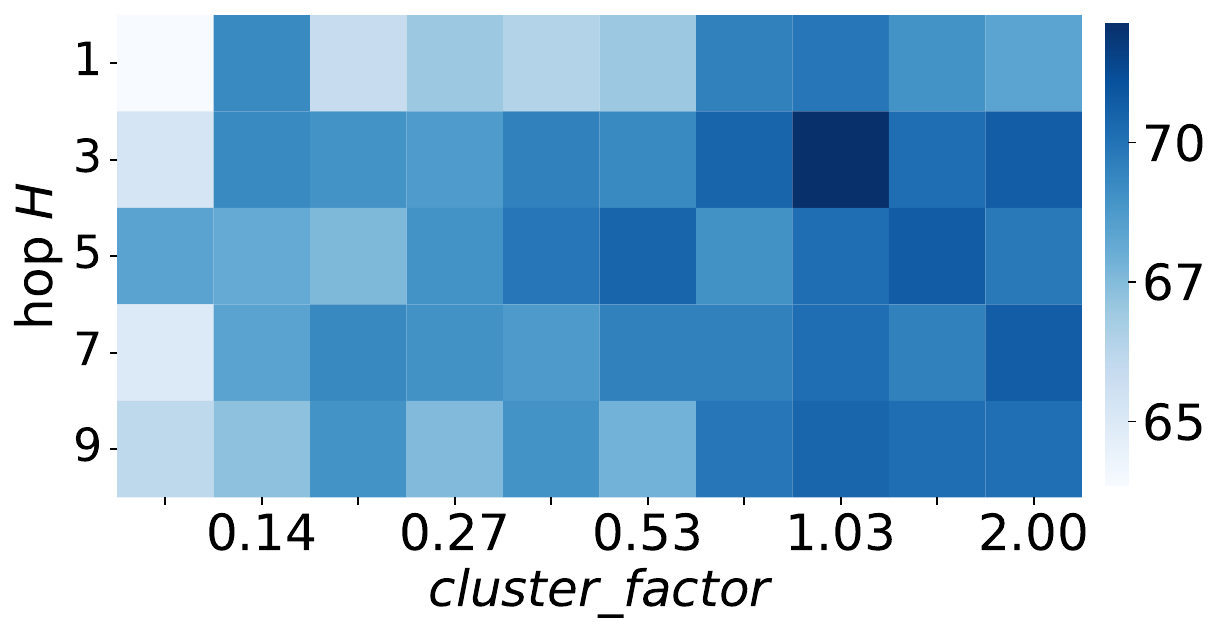}}
    % \vspace{0.5pt}
    \hspace*{\fill}

    \hspace*{\fill}
    \centering
    \subfloat[DHFR]{\includegraphics[width=.45\columnwidth]{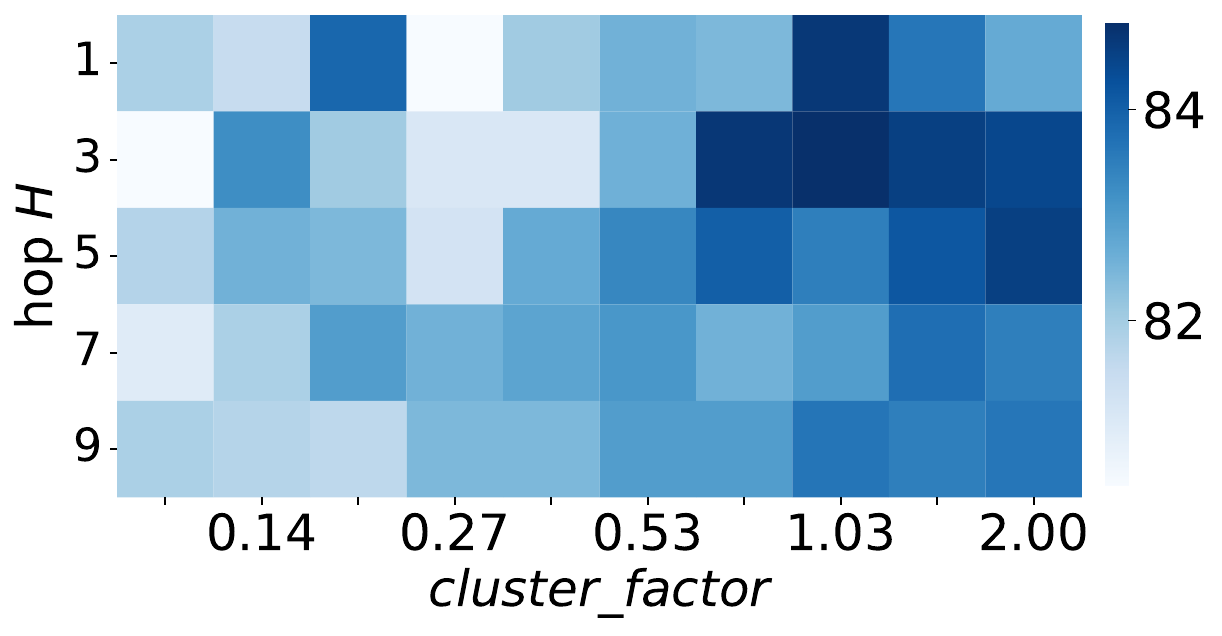}}
    \hfill
    \hfill
    \centering
    \subfloat[IMDB-B]{\includegraphics[width=.45\columnwidth]{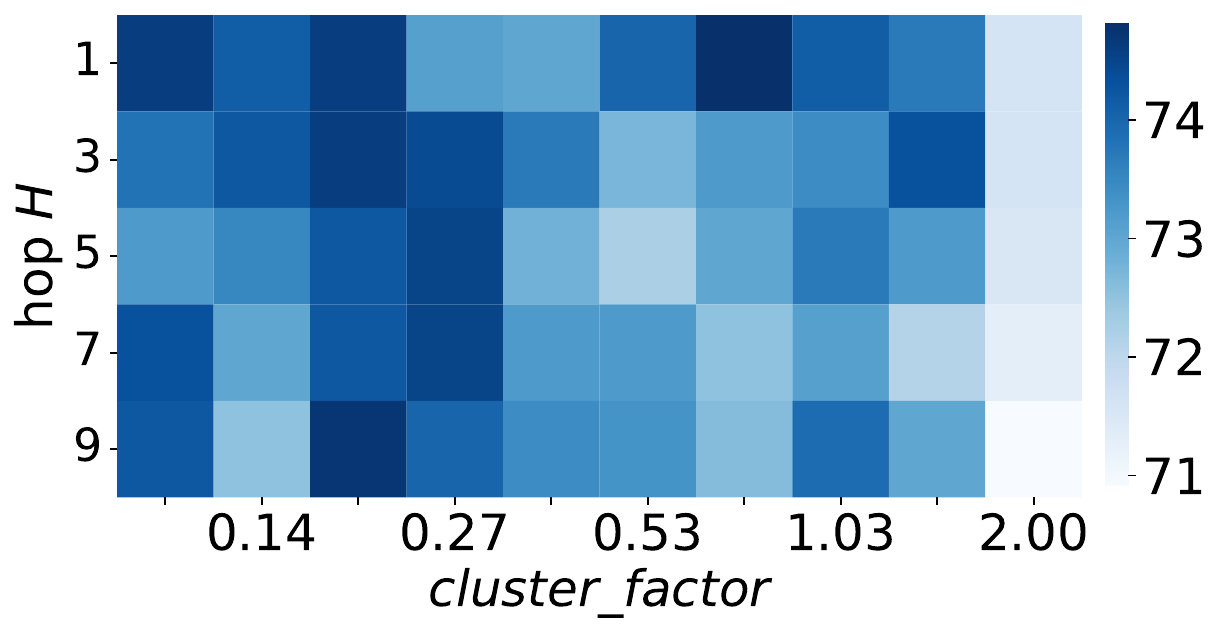}}
    \hspace*{\fill}
\caption{Paramter sensitivity analysis on \(H\) and \(cluster\_factor\) in {\namemodel}-BERT. We present the classification accuracy for combinations of \(H\in\{1,3,\ldots,9\}\) and \(cluster\_factor\in\{0.1, 0.14,\ldots,2\}\) selected from 0.1 to 2 in base 10 log scale.}
\label{fig:paramterSensitivityH_cluster}
\end{figure}

With \(b\) fixed to the best parameter in each dataset, and \(\alpha\) fixed to 0.6, we draw the heatmap of average classification accuracy over the 10 folds for all combinations of \(H \in \{1, 3, \ldots, 9\}\) and \(cluster\_factor\) varying from 0.1 to 2 in base 10 log scale. In Figure \ref{fig:paramterSensitivityH_cluster}, We can observe that {\namemodel}-BERT behaves stably across various parameter combinations, consistently achieving high performance.

Then we fix \(H\) and \(|\mathcal{C_\psi}|\) to the best parameter, and vary \(b\in \{0,1,2,3\}\) and \(\alpha \in \{0, 0.2, 0.4, 0.6, 0.8, 1\}\). Figure~\ref{fig:paramterSensitivity_ab} shows the heatmap for all combinations of \(b\) and \(\alpha\). As can be seen, the best classification accuracy may be achieved when using values from the set \(\{0, 1, 2\}\) for \(b\), with \(\alpha\) approximately around 0.6. 
% To reduce the search space, we fix $\alpha$ to 0.6 in our experiments. 
For the IMDB-B dataset, when \(\alpha\) is fixed, increasing \(b\) beyond 2 does not improve the performance. The reason is that, for datasets with a large number of edges in the graph, the encoding of slice can easily exceed the input sentence length limit (the token size) of BERT, leading to a truncation of encoding. Exploiting Large Language Models (LLM) such as GPT4 is left for future work.

\begin{figure}
    \hspace*{\fill}
    \centering
    \subfloat[KKI]{\includegraphics[width=.45\columnwidth]{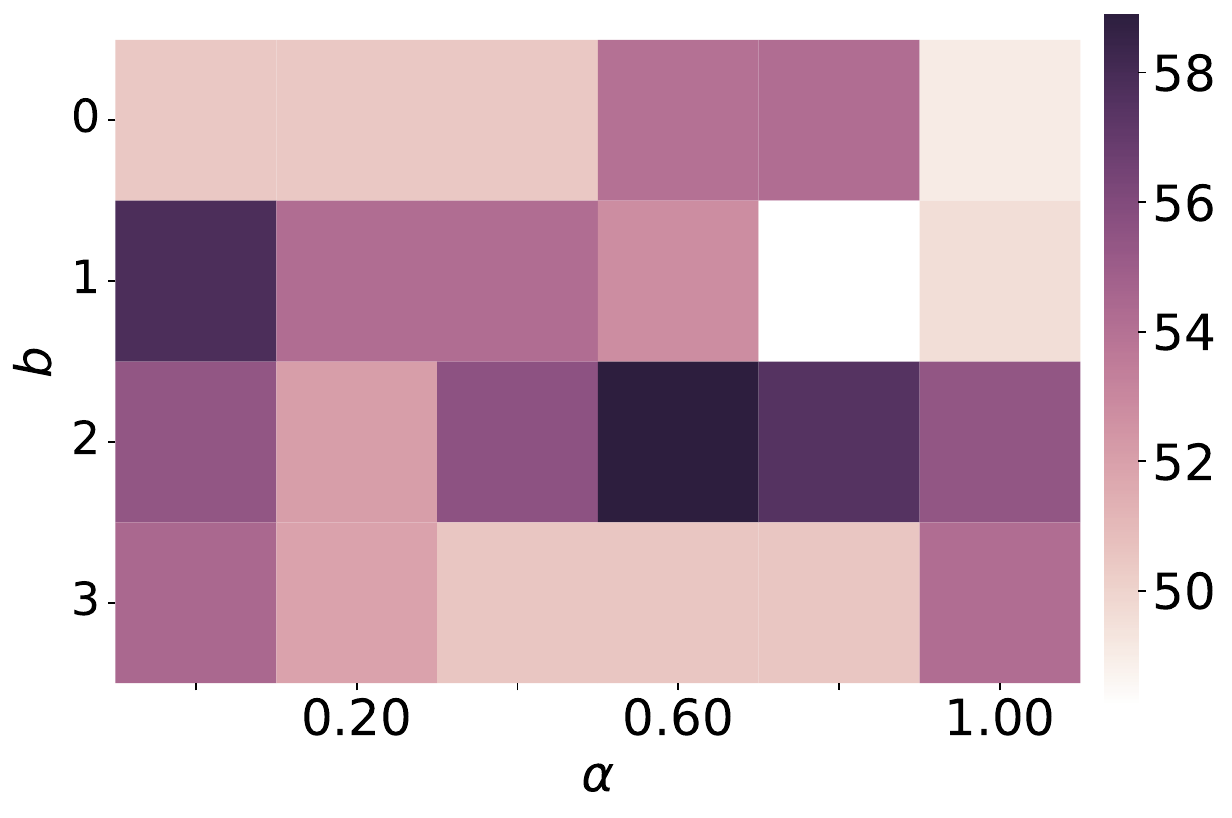}}
    \hfill
    \hfill
    \centering
    \subfloat[PTC\_FR]{\includegraphics[width=.45\columnwidth]{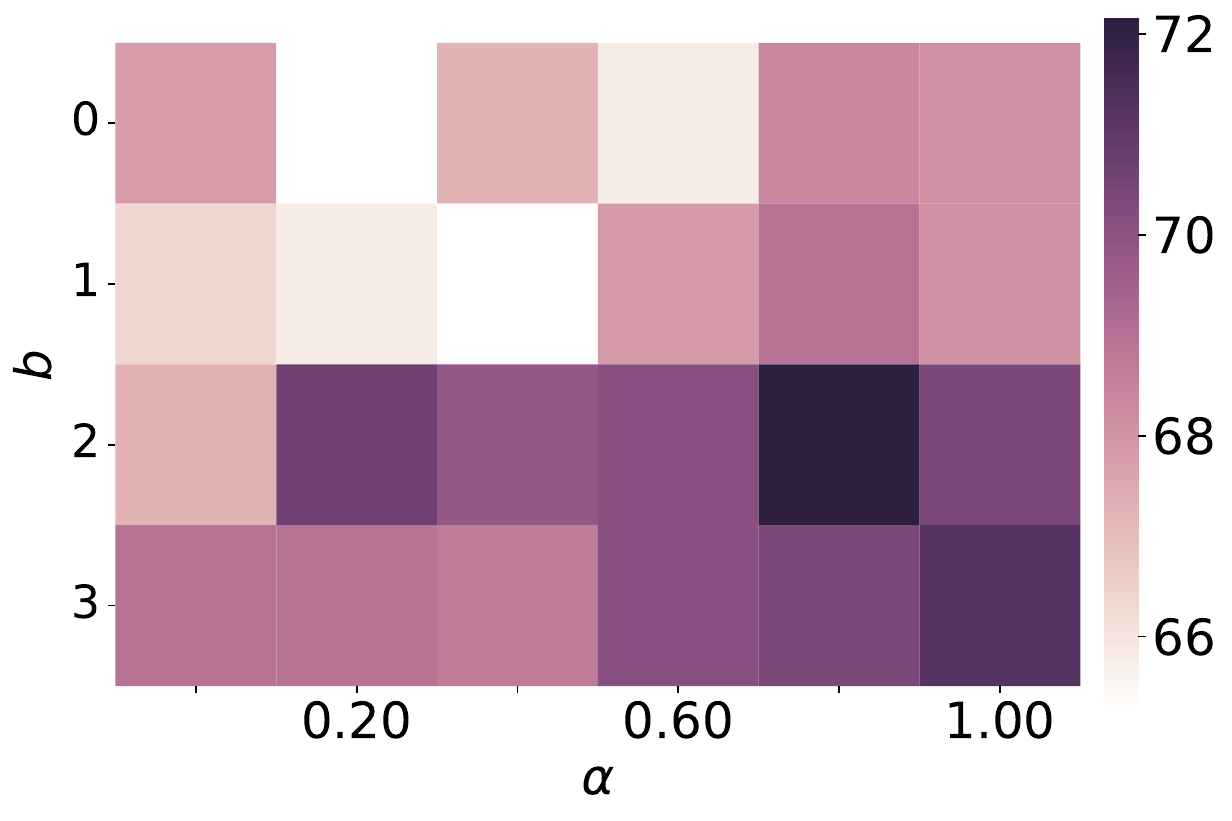}}
    % \vspace{0.5pt}
    \hspace*{\fill}

    \hspace*{\fill}
    \centering
    \subfloat[DHFR]{\includegraphics[width=.45\columnwidth]{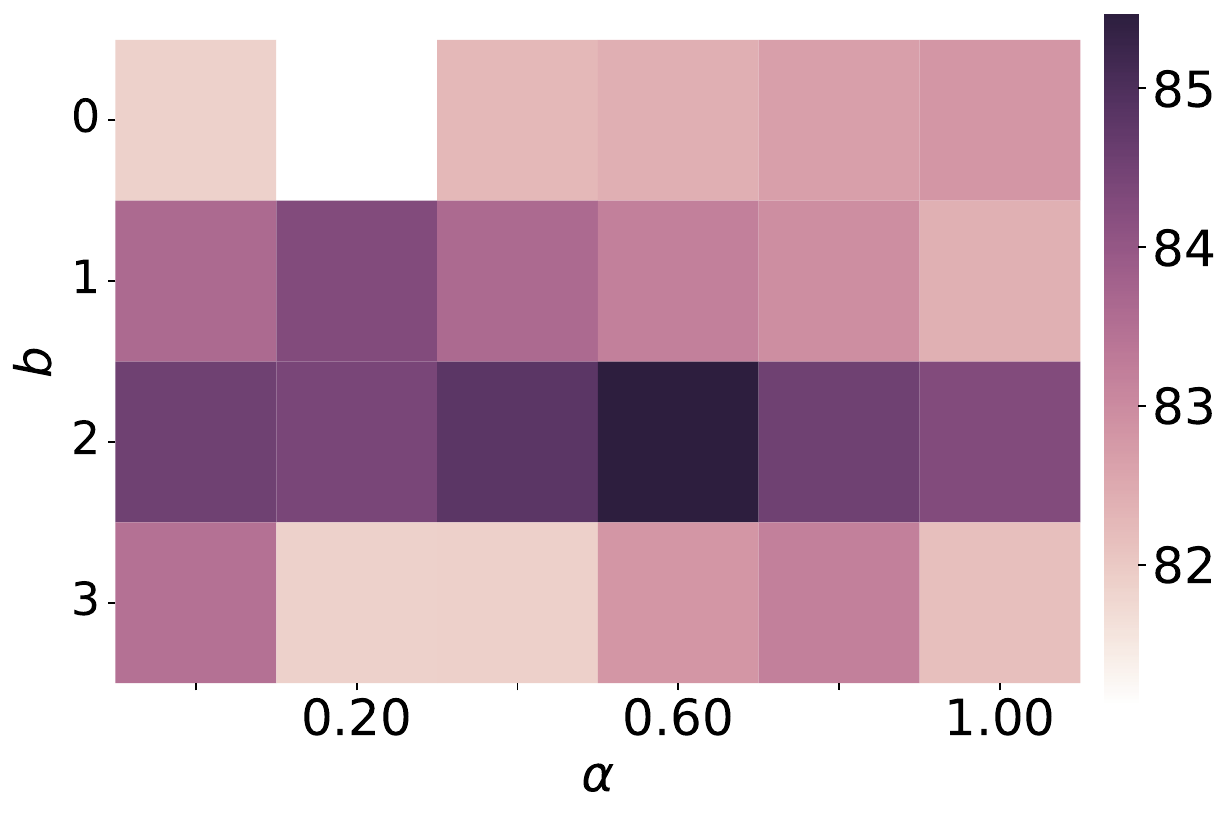}}
    \hfill
    \hfill
    \centering
    \subfloat[IMDB-B]{\includegraphics[width=.45\columnwidth]{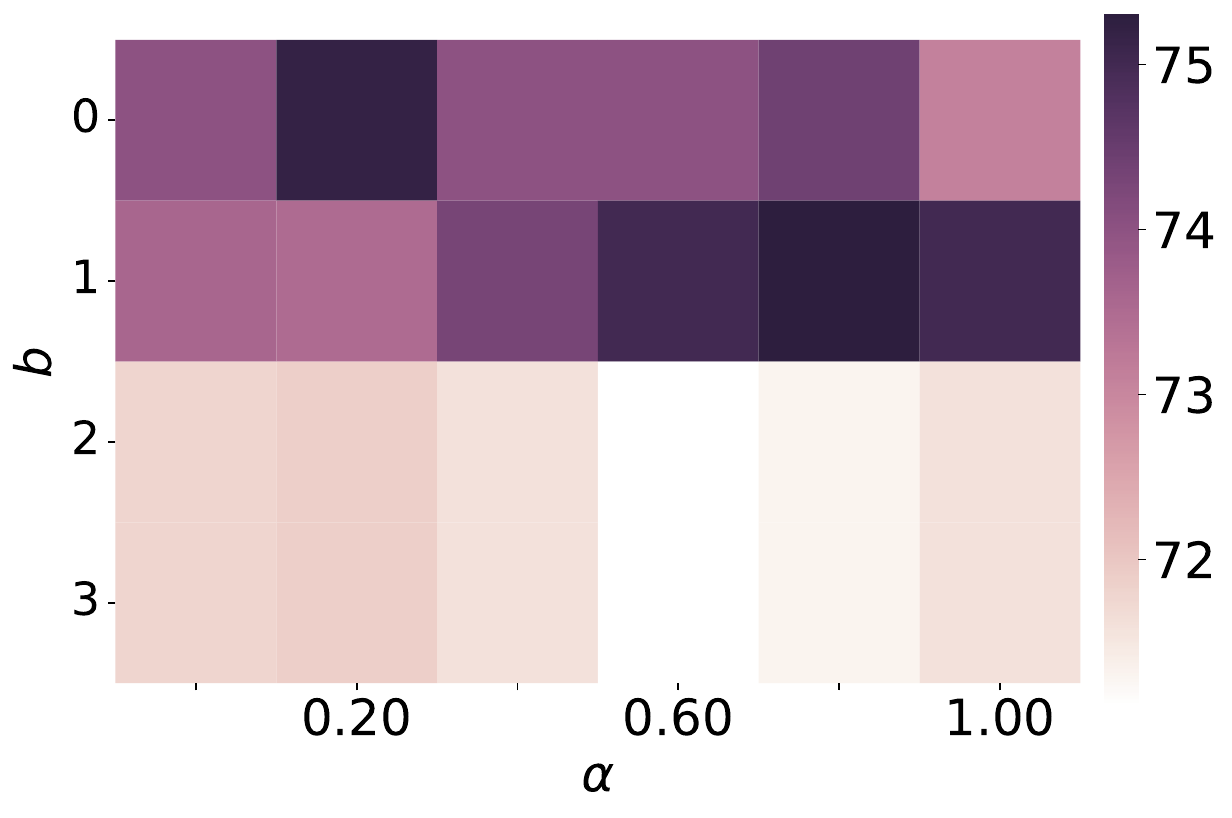}}
    \hspace*{\fill}
\caption{Parameter sensitivity analysis on \(b\) and \(\alpha\) in {\namemodel}-BERT. We present the classification accuracy for combinations of \(b\in\{0,1,2,3\}\) and \(\alpha\in\{0,0.2,0.4,\ldots,1.0\}\). }
\label{fig:paramterSensitivity_ab}
\end{figure}

\subsection{Ablation Study}

In this section, we compare {\namemodel}-BERT to its three variants: HGAK-label, DGAK-BERT, and DHW-BERT. HGAK-label is a graph kernel whose label embedding function \(g\) defined in Definition \ref{df:slice-emb} is one-hot node label mapping. DGAK-BERT is a graph kernel whose deep embeddings are learned from the nodes' (\(H+b\))-depth truncated BFS trees, not from hierarchical slices. DHW-BERT uses the 1-Wasserstein distance instead of using DAK and DGAK, whose kernel matrix is computed by summing the Laplacian kernel of 1-Wasserstein distances between graphs for each hop. The average classification accuracy via 10-fold cross-validation of them is given in Table \ref{tab:ablation}. 

The effectiveness of each component in DHGAK is demonstrated in the table. Comparisons between the performance of the three ablation variants show that HGAK-label ranks the first, indicating that the deep embedding captures the neighborhood information of slice better than direct one-hot encoding. DHW-BERT ranks the second, indicating that DAK has a better ability to capture the inherent connections of slices compared to the Wasserstein distance. Finally, DGAK-BERT exhibits the lowest performance, suggesting that the slice structure we introduced significantly contributes to the model performance of DHGAK compared to the common BFS trees.

\begin{table}
    \centering
    \resizebox{0.98\linewidth}{!}{
    \begin{tabular}{l|c|ccc}
        \toprule
         Dataset&  DHGAK-BERT&  HGAK-label&  DGAK-BERT&  DHW-BERT \\
         \midrule
         KKI&  $\mathbf{58.9\pm8.2}$ &  $55.4\pm5.2$&  $55.4\pm5.2 $& $49.3\pm12.2 $\\
         MUTAG&  $\mathbf{90.9\pm4.2}$ &  $88.2\pm4.6$&  $82.9\pm6.7$& $86.1\pm4.4$\\
         PTC\_MM&  $\mathbf{70.5\pm4.6}$&  $66.4\pm6.4$&  $67.2\pm4.9$& $66.2\pm5.7$\\
         PTC\_FR&  $\mathbf{72.1\pm5.8}$ &  $70.3\pm5.3$&  $66.1\pm5.9$& $67.8\pm3.6$\\
         DHFR&  $\mathbf{85.5\pm4.4}$ &  $82.2\pm4.9$&  $82.6\pm4.8$& $84.3\pm5.6$\\
         ENZYMES&  $\mathbf{65.3\pm7.9}$ &  $57.3\pm8.2$&  $58.8\pm9.0$& $60.3\pm5.2$\\
         IMDB-B& $\mathbf{75.3\pm2.7}$ & $74.7\pm4.0$& $70.6\pm1.2$&$73.3\pm2.2$\\
         IMDB-M& $\mathbf{52.1\pm2.4}$ & $51.7\pm2.3$& $50.7\pm3.9$&$51.6\pm2.8$\\
         COLLAB& $\mathbf{81.7\pm1.5}$ & $81.5\pm1.0$& $81.2\pm1.8$& $>24h$\\
         \midrule
         Avg. rank& 1.0& 2.6& 3.3&3.0\\
        \bottomrule
    \end{tabular}
    }
    \caption{Comparison of the average classification accuracy of DHGAK-BERT, HGAK-label, DGAK-BERT, DHW-BERT. Note that DHW-BERT applying Wasserstein distance instead of DAK fails to complete in 24 hours on large datasets. }
    % Avg.rank is the average rank of each method on each dataset except COLLAB.}
    \label{tab:ablation}
\end{table}

\subsection{Running Time}
In Table \ref{tab:runtime}, we demonstrate the running time (in seconds) of two realizations of {\namemodel} compared with some baselines on four datasets. The datasets are sorted by their sizes. As shown in the table, the fastest GKs are \(\mathcal{R}\)-convolution graph kernels WL and GAWL. The running time of DHGAK-w2v is comparable to PM. DHGAK-BERT is much slower than DHGAK-w2v, but it is still faster than WWL, WL-OA, FWL, and GraphQNTK on large datasets such as COLLAB. 
We also find that the efficiency of DHGAK behaves better on large datasets. 
% It is evident that the efficiency of DHGAK depends on the chosen NLM.
% The time bottleneck of DHGAK-BERT lies in the inference process of BERT.
It is evident that time bottleneck of DHGAK-BERT lies in the inference process of BERT.

\begin{table}
    \centering
    \resizebox{0.98\linewidth}{!}{
    \begin{tabular}{l|cccc}
    \toprule
    Method&  PTC\_MM&  PROTEINS&  DD&  COLLAB\\
    \midrule
    DHGAK-BERT&  11.58&  754.09&  5456.99&  2452.41\\
    DHGAK-w2v&  0.91&  130.81&  2708.16&  1986.64\\
    PM&  2.38&  23.62&  1068.34&  1262.27\\
    WL& 0.74& 1.52& 14.81&82.10\\
    WWL&  13.63&  238.67&  4953.69&  10044.34\\
    FWL& 7.34& 960.94& 2531.51& 7405.03\\
    WL-OA& 1.70& 174.83& 6090.02& 30743.83\\
    GAWL& 2.02& 31.12& 519.69& 3854.78\\
    GraphQNTK& 27.47& 773.25& 6434.50& $>$24h\\
    \bottomrule
    \end{tabular}
    }
    \caption{Running time (in seconds) comparison between two realizations of {\namemodel} and some baselines. DHGAK-w2v is comparable to PM. DHGAK-BERT is faster than WWL, WL-OA, FWL, and GraphQNTK on large datasets.}
    \label{tab:runtime}
\end{table}

% HUG and HUGE-w2v, in order to show the efficiency of every component of HUGE. UGE is a graph kernel derived from BFS not using the Slice. HUG use the 1-Wasserstein distance to get the graph kernel instead of using Uniform Clustering Kernel. HUGE-w2v, is a variant of using w2v for node embedding instead BERT. We run them on several benchmarks and the classification accuracy is given in Table \ref{Table}

\section{Conclusion}

In this paper, we introduce Deep Hierarchical Graph Alignment Kernels (DHGAK), a novel framework to address the prevalent issue of neglecting relations among substructures in existing $\mathcal{R}$-convolution graph kernels. We propose a positive semi-definite kernel called Deep Alignment Kernel (DAK), which efficiently aligns relational graph substructures using clustering methods on the deep embeddings of substructures learned by Natural Language Models. Then, the Deep Graph Alignment Kernel (DGAK) is constructed by applying kernel mean embedding on the feature maps of DAK. Finally, DHGAK is constructed by the summation of DGAK on slices of different hierarchies. The theoretical analysis demonstrates the effectiveness of our approach, and experimental results indicate that the two realizations of DHGAK outperform the state-of-the-art graph kernels on most datasets.

% In this paper, we introduce Deep Hierarchical Graph Alignment Kernels (DHGAK), a framework to address the prevalent issue of neglecting relations among substructures in existing $\mathcal{R}$-convolution graph kernels. We propose a positive semi-definite kernel called Deep Alignment Kernel (DAK), which efficiently aligns relational graph substructures using clustering methods on the deep embeddings of substructures learned by Natural Language Models. Then, the Deep Graph Alignment Kernel (DGAK) is constructed by applying kernel mean embedding on the feature maps of DAK. Finally, DHGAK is constructed by the summation of DGAK on slices of different hierarchies. The theoretical analysis demonstrates the effectiveness of our approach, and experimental results indicate that the two realizations of DHGAK outperform the state-of-the-art graph kernels on most datasets.

\section*{Acknowledgments}
We thank the anonymous reviewers for their valuable and constructive comments. This work was supported partially by the National Natural Science Foundation of China (grants \# 62176184 and 62206108), the National Key Research and Development Program of China (grant \# 2020AAA0108100), and the Fundamental Research Funds for the Central Universities of China.

% \clearpage

\bibliographystyle{named}
\bibliography{ijcai24}

\clearpage

\clearpage
\appendix
\section{Proofs of Theorems}

\subsection{Proof of Theorem \ref{th:semi-pos-DAK}}
\label{proof:DAK_pos}
\begin{proof}
    For the Deep Alignment Kernel (DAK), the \(\kappa_h^b(\cdot, \cdot)\) is symmetric. For any point \(\mathbf{v}_1, \mathbf{v}_2,\ldots, \mathbf{v}_N \in \mathbf{X}_h^b \) and any vector \(\mathbf{x} \in \mathbb{R}^N\), we denote \(\mathbf{K}\) as the Gram matrix of \(\mathbf{v}_1, \mathbf{v}_2,\ldots, \mathbf{v}_N\): \(\mathbf{K} = [\kappa_h^b(\mathbf{v}_i, \mathbf{v}_j)]_{N\times N}\). We have:
    \begin{equation}
        \begin{split}
            \mathbf{x}^T \mathbf{K} \mathbf {x} &= \sum_{i,j =1}^{N} x_i x_j \kappa_h^b(\mathbf{v}_i, \mathbf{v}_j) \\
            &= \sum_{i,j =1}^{N} x_i x_j 
            \mathbb{E}_{\psi \sim U(\Psi)}[\mathbb{I}(\mathbf{v}_i,\mathbf{v}_j\in \mathcal{C}|\mathcal{C}\in \mathcal{C}_\psi)] \\
            &= \sum_{i,j =1}^{N} x_i x_j \mathbb{E}_{\psi \sim U(\Psi)}[
            \sum_{k=1}^{| \mathcal{C}_\psi |} \mathbb{I}(\mathbf{v}_i\in \mathcal{C}_{\psi,k} ) \mathbb{I}(\mathbf{v}_j\in \mathcal{C}_{\psi,k} )] \\
            &= \mathbb{E}_{\psi \sim P(\Psi)}[ \sum_{k=1}^{| \mathcal{C}_\psi |} 
            \sum_{i=1}^{N} x_i \mathbb{I}(\mathbf{v}_i\in \mathcal{C}_{\psi,k} ) \sum_{j=1}^{N} x_j \mathbb{I}(\mathbf{v}_j\in \mathcal{C}_{\psi,k} )] \\
            & = \mathbb{E}_{\psi \sim P(\Psi)}[ \sum_{k=1}^{| \mathcal{C}_\psi |} 
            \Vert \sum_{i=1}^{N} x_i \mathbb{I}(\mathbf{v}_i\in \mathcal{C}_{\psi,k} ) \Vert ^ 2
            ] \geq 0
        \end{split}
    \end{equation}
    Thus, the original DAK is positive semi-definite. For the approximation of DAK:
   \begin{equation}
         \kappa_h^b(\mathbf{x}, \mathbf{y}) \simeq \frac{1}{T | \Psi | } 
        \sum_{\psi \in \Psi} \sum_{t=1}^{T} 
        \sum_{i=1}^{|\mathcal{C}_\psi |} \mathbb{I}(\mathbf{x}\in \mathcal{C}^{(t)}_{\psi,i}) \mathbb{I}(\mathbf{y}\in \mathcal{C}^{(t)}_{\psi,i})
\end{equation}
    One can similarly prove that it is also positive semi-definite.
\end{proof}
   
\subsection{Proof for Theorem \ref{th:transitive}}
\label{proof:transitive}    \begin{proof}
        By the definition of alignment operation \(\mathcal{M}_{h,\psi}^{b,t}\), if \((\mathbf{x}, \mathbf{y})\) are aligned and \((\mathbf{y}, \mathbf{z})\) are also aligned, which means \(\mathbf{x}\), \(\mathbf{y}\) and \(\mathbf{z}\) belong to the same cluster, i.e., \(\mathbf{x}, \mathbf{y}\) and \(\mathbf{z}\) are all aligned. Thus \(\mathcal{M}_{h,\psi}^{b,t}(\mathbf{x}, \mathbf{z})= \mathcal{M}_{h,\psi}^{b,t}(\mathbf{x}, \mathbf{y}) \mathcal{M}_{h,\psi}^{b,t}(\mathbf{y}, \mathbf{z})\). 
\end{proof}
    
% \subsection{Proof of Theorem \ref{th:semi-pos-DHGAK}}
% \label{proof:semi-pos-DHGAK}
% \begin{proof}
% The DHGAK is defined as:
% \begin{equation}
% \mathcal{K}(G_1,G_2)=\sum_{h=1}^H\mathcal{K}_h^b(G_1,G_2)
% \end{equation}
% Here \(\mathcal{K}_h^b(G_1, G_2)\) is Deep Graph Alignment Kernel:
% \begin{equation}
% \begin{split}
% \mathcal{K}_h^b(G_1,G_2)&=\langle \frac{1}{|\mathbf{X}_h^b|}\sum_{\mathbf{x}\in\mathbf{X}_h^b}\phi_h^b(\mathbf{x}), \frac{1}{|\mathbf{Y}_h^b|}\sum_{\mathbf{y}\in\mathbf{Y}_h^b}\phi_h^b(\mathbf{y}) \rangle\\
% &=\frac{1}{|\mathbf{X}_h^b||\mathbf{Y}_h^b|}\sum_{\mathbf{x}\in\mathbf{X}_h^b}\sum_{\mathbf{y}\in\mathbf{Y}_h^b}\kappa_h^b (\mathbf{x},\mathbf{y})
% \end{split}
% \end{equation}
% Note that \(\phi_h^b(\cdot)\) is the feature map of Graph Alignment Kernel, thus  \(\phi_h^b(\cdot) \in\) its RKHS \(\mathcal{H}\). And \(\frac{1}{|\mathbf{X}_h^b|}\sum_{\mathbf{x}\in\mathbf{X}_h^b}\phi_h^b(\mathbf{x})\) is the linear combination of vectors in \(\mathcal{H}\), also lays in \(\mathcal{H}\). Therefore Deep Graph Alignment Kernel \(\mathcal{K}_h^b(G_1,G_2)\) is the inner product between two vector in \(\mathcal{H}\), that is positive semi-definite. DHGAK is sum of positive semi-definite kernels, which is also positive semi-definite.
% \end{proof}

\subsection{Proof for Linear Separation Property of {\namemodel}}
\label{proof:LinearSeparationPropertyOfDHGAK}
% The performance of HUGE is depended on the selection of the set of clustering method \(\Psi \subset \mathfrak{M} \). In this section, we will show there exits some subsets of  \(\mathfrak{M}\) that completely linearly separate \(\hat{\Phi}_{U}(G)\) for any training dataset \(\{(G_i, y_i)\}_{i=1}^{N}\) theoretically. 

Given a graph dataset \(\mathcal{D}=\{(G_i, y_i)\}_{i=1}^{N}\), let \(\Psi_o\) denote the set of clustering methods \(\{\psi_{j}\}_{j=1}^{|\Psi_o|}\), where \(\psi_j: \mathbf{X}_h^b \to \mathcal{C}_{\psi_{j}}\) separates each point \(\mathbf{x}_h^b\) (the embedding of a slice extracted around the node $v$) to the cluster \(\mathcal{C}_{\psi_j, k}\) if and only if \(v \in V_k\) belonging to \(G_k\). First, we give a Lemma as follows:

\begin{lemma}
    Consider unit vectors \(\{\mathbf{e}_i\}_{i=1}^N \) in space \(\mathbb{R}^N\), we arbitrarily assign labels \(\mathcal{Y}=\{+1, -1\}\) to them to construct datasets \(\{(\mathbf{e}_i, y_i)\}_{i=1}^N\). There exists a family of hyper-planes in \(\mathbb{R}^N\) that separates samples with label \(+1\) from samples with label \(-1\).
    \label{lemma1}
\end{lemma}
\begin{proof}
    First, we define the linear separation as the same as \cite{SVM}. Let \(f(\mathbf{x}) = 0\) be a hyperplane in \(\mathbb{R}^N\), where \(f(\mathbf{x}) = \mathbf{w}^T\mathbf{x} + b\). The set of datasets \(\mathcal{D}=\{(\mathbf{e}_i, y_i)\}_{i=1}^N\) and \(y_i\in\{+1,-1\}\), is said to be linearly separable if there exists a vector \(\mathbf{w}\) and a scalar \(b\) satisfying:
    
    \begin{equation}
        y_i (\mathbf{w}^T \mathbf{e}_i + b) > 0, \ \ i=1,\ldots,N
    \end{equation}
    Then we consider the function defined as follows:
    \begin{equation}
    \begin{array}{cc}
         & f^\star(\mathbf{x}) = 
            \begin{vmatrix}
            x_1 - c_1 & x_2 & \cdots & x_N \\
                - c_1 & c_2 & \cdots & 0 \\
                \vdots & \vdots & \ddots & \vdots \\
                - c_1 & 0 & \cdots & c_N \\
            \end{vmatrix} \\
            & = \begin{bmatrix}
             c_2c_3\cdots c_N\\
             c_1c_3\cdots c_N\\
                \vdots \\
             c_1c_2\cdots c_{N-1}
            \end{bmatrix} ^ T
            \begin{bmatrix}
             x_1\\
             x_2\\
                \vdots \\
             x_N
            \end{bmatrix}
            - 
            c_1c_2\cdots c_N =: {\mathbf{w}^\star}^T \mathbf{x} + b ^\star
    \end{array}
\end{equation}
    Here \(\mathbf{x} = [x_1, \cdots, x_N]^T\) and
    \[
    c_i =
    \begin{cases}
      1 + \text{sign}(y_i) \beta_i  &\text{if mod}(N,2)=0 \\
      1 - \text{sign}(y_i) \beta_i & \text{if mod}(N,2)=1 \\
    \end{cases}
    \]
    for any \(\beta_i>0\) that makes \(c_i > 0\).
    The \(f^\star (\mathbf{x})\) can linearly separate dataset \(\mathcal{D}=\{(\mathbf{e}_i, y_i)\}_{i=1}^N\).
    
\end{proof}

We take \(N=3\) for an example, 
\[
f^\star(\mathbf{x})= c_2c_3x_1 + c_1c_3x_2 + c_1c_2x_3 - c_1c_2c_3
\]
For \(i=1\), we have
\[
\begin{split}
y_1f(\mathbf{e_1}) &= y_1[c_2c_3 - c_1c_2c_3] \\
&= y_1[(1-(1-\text{sign}(y_1)\beta_1)) c_2c_3] \\
&= \text{sign}^2(y_1)c_2c_3 \beta_1 |y_1| > 0
\end{split}
\]
Similarly, for \(i = 2\) and \(3\), we obtain similar inequalities.

Then we give the proof of Theorem \ref{th:linearSeparationPropertyOfDHGAK}.

\begin{proof}
    Consider a graph label set with only two classes \(\mathcal{Y} \in \{+1, -1\}\), as we focus on binary classification. This simplification is made, even when dealing with multi-class \(\mathcal{Y}\), by treating one class against all others when referring to linear separability. The feature map of DHGAK, defined by Equation \ref{eq:DHGAK}, is derived as follows:
    \begin{equation}
    \begin{split}
         &\mathcal{K}(G_1,G_2) =\sum_{h=1}^H\mathcal{K}_h^b(G_1,G_2)  \\
         &= \sum_{h=1}^H \langle \frac{1}{|\mathbf{X}_h^b|}\sum_{\mathbf{x}\in\mathbf{X}_h^b}\phi_h^b(\mathbf{x}), \frac{1}{|\mathbf{Y}_h^b|}\sum_{\mathbf{y}\in\mathbf{Y}_h^b}\phi_h^b(\mathbf{y}) \rangle \\
         &=\langle \parallel_{h=1}^H \frac{1}{|\mathbf{X}_h^b|}\sum_{\mathbf{x}\in\mathbf{X}_h^b}\phi_h^b(\mathbf{x}), \parallel_{h=1}^H \frac{1}{|\mathbf{Y}_h^b|}\sum_{\mathbf{y}\in\mathbf{Y}_h^b}\phi_h^b(\mathbf{y}) \rangle\\ 
         &=: \langle \hat{\Phi}(G_1), \hat{\Phi}(G_2) \rangle\\
    \end{split}
    \end{equation}
     Because each \(\psi_{j} \in \Psi_o\) separates each point \(\mathbf{x}\in \mathbf{X}_h^b\) (the embedding of the slice extracted from the node $v$) to the cluster \(\mathcal{C}_{\psi_h, k}\) if and only if \(v \in V_k\) belonging to \(G_k\), the cluster indicator at the \(t\)-th experiment under \(\psi_j\) is presented as:
    \[
    \begin{split}
    \mathbbm{1}^{(t)}_{\psi_j}(\mathbf{x}) &= [\overbrace{0,\ldots , 1,\ldots,0}^{N}]^T \\
    & \phantom{...............}\uparrow \text{index} \ k: v\in V_k \text{ of } G_k
    \end{split}
    \]
    
    The explicit form of the feature map of DHGAK under \(\Psi_o\) is derived as follows:
    \[
    \begin{split}
        \hat{\Phi}(G_k) &= \frac{1}{\sqrt{T| \Psi_o |}}   \parallel_{h=1}^H \sum_{\mathbf{x}\in\mathbf{X}_h^b} \frac{1}{|\mathbf{X}_h^b|} \parallel_{\psi \in \Psi_o}\parallel_{t=1}^T\mathbbm{1}^{(t)}_\psi(\mathbf{x})  \\
        & = \frac{1}{\sqrt{T| \Psi_o |}}  \parallel_{h=1}^H \parallel_{\psi \in \Psi_o} \parallel_{t=1}^T [\overbrace{0,\ldots , 1,\ldots,0}^{N}]^T \\
        & \phantom{...................................................}\uparrow \text{index} \ k \\
        &= \frac{1}{\sqrt{T| \Psi_o |}} \mathbf{e}_k \otimes \mathbf{1}_{(TH|\Psi_o|)\times1} \in \mathbb{R}^{NTH|\Psi_o|}
    \end{split}
    \]
    Here, \(\mathbf{e}_k\) represents the \(k\)-th unit vector in \(\mathbb{R}^N\),  \(\mathbf{1}_{(TH|\Psi_o|)\times1}\) denotes the full-one vector and \(\otimes\) represents Kronecker product. 
    
    By Lemma \ref{lemma1},  we have \(y_kf^\star (\mathbf{e}_k) > 0\) for \(k=1,2,\ldots, N\). Let \(\hat{\mathbf{w}} = \mathbf{w}^\star \otimes \mathbf{1}_{(TH|\Psi_o|)\times1} \), \(\hat{b}=TH|\Psi_o| b^\star\) and \(\hat{f}(\mathbf{z}) = \hat{\mathbf{w}}^T \mathbf{z} + \hat{b} \) for \(\mathbf{z}\in\mathbb{R}^{NTH|\Psi_o|}\), then we have:
    \begin{equation}
    \begin{split}
        &y_k\hat{f}(\hat{\Phi}(G_k)) = \frac{1}{\sqrt{T| \Psi_o |}} (\mathbf{w}^\star \otimes \mathbf{1}_{(TH|\Psi_o|)\times1})^T \\
        &(\mathbf{e}_k \otimes \mathbf{1}_{(TH|\Psi_o|)\times1}) + TH|\Psi_o| b^\star \\
        &= \frac{1}{\sqrt{T| \Psi_o |}}({\mathbf{w}^\star}^T \mathbf{e}_k \otimes \mathbf{1}_{(TH|\Psi_o|)\times1}^T \mathbf{1}_{(TH|\Psi_o|)\times1} ) + TH|\Psi_o| b^\star \\
        &= H\sqrt{T|\Psi_o|}({\mathbf{w}^\star}^T \mathbf{\mathbf{e}}_k + b ^\star) > 0
    \end{split}
    \end{equation}

Hence, the hyperplanes \(\hat{f}(\mathbf{z}) = \hat{\mathbf{w}}^T\mathbf{z} + \hat{b} = 0\) with \(\beta_i > 0, i=1,\ldots,N\), can linearly separate the feature maps of DHGAK, denoted as \(\hat{\Phi}(G)\), for graph labels \(+1\) and \(-1\). Due to the randomness of the labels represented by \(+1\), there exists a family of hyperplanes \(\{\hat{f}_1, \ldots, \hat{f}_{|\mathcal{Y}|}\}\) that make the feature map of DHGAK on a graph linearly separable from those of other graphs with different labels.

\end{proof}

\section{Pseudo-codes}
\label{psucodes}

\subsection{Pseudo-code for the Encoding of Slice}

The pseudo-code for the encoding of slice is shown in Algorithm \ref{alg:slice_encoding}. For a special case, at line 4, if two nodes have the same eigenvector centralities, we sort them by their label values and node indices.

% The pseudo-code for the encoding of Slice and the whole DHGAK are shown in Algorithm \ref{alg:slice_encoding} and Algorithm \ref{alg:DHGAK} respectively.

\label{psucode:EncodingSlice}
\begin{algorithm}[!htbp]
\caption{The encoding of slice}
\label{alg:slice_encoding}
\floatname{algorithm}{Procedure}
\begin{algorithmic}[1]
\renewcommand{\algorithmicrequire}{\textbf{Input:}}
\renewcommand{\algorithmicensure}{\textbf{Output:}}
\REQUIRE Graph \(G=(\mathcal{V},\mathcal{E},l)\), node \(v\in \mathcal{V}\), hop \(h\) and width $b$.
\ENSURE Slice encoding $S_h^b(v)$.

\STATE \(EC \gets \) GetEigenCentrality(G); /* Get the Eigenvector Centrality */
\STATE \(S_h^b(v) \gets []\);
\STATE \(N_h(v) \gets\) GetBFSLeaves(\(G,v,h\));
\STATE \(N_h(v) \gets N_h(v)\).sort(key=lambda \(x: EC(x)\));        /* sorting $N_h(v)$ by eigenvector centrality */

\FORALL{ \(u \in N_h(v)\)}
    \STATE \(S^b(u) \gets []\); 
    \FOR{\(i \gets 0\) \TO \(b\)}
        \STATE \(N_i(u) \gets\) GetBFSLeaves(\(G,u,i\));
        \STATE \(N_i(u) \gets N_i(u)\).sort(key=lambda \(x: EC(x)\));
        \STATE \(S^b(u)\).extend(\(N_i(u)\));
    \ENDFOR
    \STATE \(S_h^b(v)\).extend(\(S^b(u)\));
\ENDFOR

\RETURN $l(S_h^b(v))$;
\end{algorithmic}
\end{algorithm}

\subsection{Pseudo-code for DHGAK}

The pseudo-code for DHGAK is shown in Algorithm \ref{alg:DHGAK}.

\label{psucode:DHGAK}
\begin{algorithm}[!htbp]
    \caption{Generate kernel matrix of DHGAK}
\label{alg:DHGAK}
\floatname{algorithm}{Procedure}
    \begin{algorithmic}[1]
    \renewcommand{\algorithmicrequire}{\textbf{Input:}}
    \renewcommand{\algorithmicensure}{\textbf{Output:}}
    \REQUIRE  Undirected labeled graph dataset \(\mathcal{D}=\{G_i=(\mathcal{V}_i, \mathcal{E}_i)\}_{i=1}^{N}\), the maximum hop \(H\) and width \(b\) for hierarchical slices, label embedding function (learned by NLMs) \(g\), decay factor \(\alpha\), the set of clustering method \(\Psi\) and experiment time \(T\).
    \ENSURE The kernel matrix of DHGAK $\mathbf{K} \in \mathbb{R}^{N\times N}$.

    \STATE /* Get the deep embeddings of hierarchical slices */ 
    \STATE \(\mathbf{X} \gets []\);
    \FOR{ \(h \gets 0\) \TO \(H\)}
        \STATE \(\mathbf{X}_h^b\) \(\gets []\);
        \FORALL{ \(G=(\mathcal{V}, \mathcal{E}) \in \mathcal{D}\)}
            \STATE \(\mathbf{x}_h^b\gets []\);
            \FORALL{\(v\in V\)}
                \STATE /* Get the slice encoding by Algoirthm \ref{alg:slice_encoding} */
                \STATE \(S_h^b(v) \gets\) SliceEncoding(\(G, v, h, b\));
                \STATE \(\mathbf{x}_h^b \text{.append(} \sum_{s\in S_h^b} g (s)\));
            \ENDFOR
            \STATE \(\mathbf{X}_h^b \text{.expand(} \mathbf{x}_h^b\));  /* Concatenate \(x^b_h\) for all graphs */
        \ENDFOR
        \STATE \(\mathbf{X}.\text{append(}\mathbf{X}_h^b\));
    \ENDFOR
    \FOR{\(h\gets 1\) \TO H}
        \STATE \(\mathbf{X}[h] \gets \alpha\mathbf{X}[h-1] + \mathbf{X}[h]\)
    \ENDFOR    
    \STATE /* Compute the kernel matrix of DHGAK */
    \STATE \(\mathbf{K} \gets \text{zeros}(N, N)\);
    % \STATE \(\phi \gets []\);
    \FOR{\(h\gets 0\) \TO \(H\)}
        \STATE \(\phi_h^b \gets []\), \(\hat{\Phi}_h^b\gets []\);
        \FORALL{\(\psi\in\Psi\)}
            \FOR{ \(t\gets 1\) \TO \(T\)}
                \STATE \(\mathcal{C}_\psi^{(t)} \gets \psi(\mathbf{X}[h])\);
                \STATE \(\mathbbm{1}_\psi^{(t)} \gets\) GetClusterIndictors(\(\mathcal{C}_\psi^{(t)}\));
                \STATE \(\phi_h^b \text{.expand(} \frac{1}{\sqrt{T|\Psi|}} \mathbbm{1}_\psi^{(t)}\));
            \ENDFOR
        \ENDFOR
        \STATE /* Using kernel mean embedding */
        \FORALL{\(G=(\mathcal{V}, \mathcal{E}) \in \mathcal{D}\)}
            \STATE \(\Phi \gets \text{zeros}(|\mathcal{V}|, 1)\);
            \FORALL{\(v\in \mathcal{V}\)}
                \STATE \(\Phi \gets \Phi + \frac{1}{|V|}\phi_h^b(\mathbf{x}_h^b(v))\);
            \ENDFOR
            \STATE \(\hat{\Phi}_h^b \text{.append(} \Phi\)); 
        \ENDFOR
        \STATE \(\mathbf{K} \gets \mathbf{K} + \langle\hat{\Phi}_h^b, \hat{\Phi}_h^b \rangle\);
    \ENDFOR

    % \STATE  \(\mathcal{E} \gets []\);
    % \FORALL{\(G \in \mathcal{D}\)}
    %     \STATE emb \( \gets []\);
    %     \FORALL{ \(v \in V\)}
    %         \STATE h\_emb \(\gets []\);
    %         \FOR {\(i \gets 0\) \TO \(h\)}
    %             \STATE /* Get the embedding of Slice by Eq. \ref{eq:node_emb} */
    %             \STATE \(e_i^b(v) = SliceEmbedding(v, G, i, b, l_v)\) ;
    %             \STATE h\_emb.append(\( e_h^b(v)\));
    %         \ENDFOR
    %         \STATE emb.append(h\_emb);
    %     \ENDFOR
    %     \STATE \(\mathcal{E}\).append(emb);
    % \ENDFOR
    % \STATE \(\Phi_{U} = []\); \(\hat{\Phi}_{U} = []\);  /* The feature map of node and graph */
    % \FORALL{\(v \in \bigcup_{i=1}^N V_i\)}
    %     \STATE /* Get UCK-indiced feature of node by Eq. \ref{eq:node_emb}.  */
    %     \STATE \(\Phi_{U} (v) \gets\) UniformClusteringFeat(\(v, \mathcal{E}, \Psi, T\));  
    %     \STATE \(\Phi_{U}.append(\Phi_{U} (v))\);
    % \ENDFOR
    % \FORALL{\(G \in \mathcal{D}\)}
    %     \STATE  /* Get mean feature of graph by Eq. \ref{eq:uck-graph}. */
    %     \STATE  \(\hat{\Phi}_{U}(G)= \) GraphFeatureMap(\(\Phi_{U}\)); 
    %     \STATE \(\hat{\Phi}_{U}.append(\hat{\Phi}_{U}(G))\);
    % \ENDFOR
    % \STATE \(K_{U} = \langle \hat{\Phi}_{U}, \hat{\Phi}_{U} \rangle\)     /* Get kernel by Eq. \ref{eq:huge-ker}.  */
    \RETURN \(\mathbf{K} \);
    \end{algorithmic}
\end{algorithm}

\subsection{Time Complexity Analysis }

\label{sec:TimeComplexityAnalysis}

In Algorithm~\ref{alg:DHGAK}, lines 1-18 generate the deep embeddings of slices.
The time complexity of constructing a BFS tree is bounded by $\mathcal{O}(n+e)$, where $n$ and $e$ are the average number of nodes and edges in the graph, respectively. We construct the truncated BFS tree rooted at each node in a graph to a depth of $max(h,b)$ and store them. Note that the BFS tree contains at most $n$ nodes. Therefore, the time complexity of encoding the slice for a graph is $\mathcal{O}(n(n+e) + n\cdot n \cdot \mathcal{O}(1))$. The first term represents the time complexity of constructing truncated BFS trees to a depth of $max(h,b)$ for every node in the graph, and the second term represents the time complexity of retrieving $b$-depth truncated BFS trees from the stored trees for each leaf node of the $h$-depth BFS trees.
The time complexity of computing Eigenvector Centrality (EC) is approximately $\mathcal{O}(n^2)$, and EC only needs to be computed once for each graph. Therefore, the time complexity of generating the deep embeddings is bounded by $\mathcal{O}(HN(n^2 + (n(n+e) + n^2)\mathcal{T}_{NLM})) \simeq \mathcal{O}(HN(n^2+ne)\mathcal{T}_{NLM})$, where $H$ is the maximum hop, $N$ is the size of the dataset, and $\mathcal{T}_{NLM}$ is the time complexity of the selected NLM for inference per token.

Lines 19-29 compute the feature maps of DAKs by applying each $\psi \in \Psi$ hierarchically on the deep embeddings. The time complexity of this process is approximate $\mathcal{O}(HT|\Psi|\mathcal{T}_C(Nn))$, where $T$ is the number of experiments for each clustering method, and $\mathcal{T}_C(Nn)$ represents the average time complexity of clustering methods.

Lines 30-40 calculate the kernel matrix of DHGAK using kernel mean embedding, with a time complexity of $\mathcal{O}(HNn)$. Consequently, the worst-case time complexity of DHGAK is bounded by $\mathcal{O}(HN(n^2 + ne)\mathcal{T}_{NLM} + HT|\Psi|\mathcal{T}_C(Nn) + HNn) \simeq \mathcal{O}(HN(n^2 + ne)\mathcal{T}_{NLM} + HT|\Psi|\mathcal{T}_C(Nn))$.

\section{Details of Experiments}

\subsection{Detail of Datasets}

The datasets we used in the experiments are from four categories: chemical compound datasets, molecular compound datasets, brain network datasets, and social network datasets. 

The chemical compound datasets include MUTAG \cite{MUTAG}, BZR \cite{BZR_COX2_DHFR}, COX2 \cite{BZR_COX2_DHFR}, DHFR \cite{BZR_COX2_DHFR}, and NCI1 \cite{NCI1_NCI109}.
% , and NCI109 \cite{NCI1_NCI109}. 
The Molecular compound datasets include PTC \cite{PTC} (comprising male mice (MM), male rats (MR), female mice (FM), and female rats (FR)), ENZYMES \cite{ENZYMES_PROTEINS}, PROTEINS \cite{ENZYMES_PROTEINS}, and DD \cite{DD}. 
The brain network dataset includes KKI \cite{KKI}. 
The Social network datasets \cite{DGKandIBDM_Reddit} contain IMDB-B, IMDB-M, and COLLAB. The statistics of these datasets are given in Table \ref{table:datasets}.

\begin{table}[!htbp]
    \centering
    \resizebox{0.45\textwidth}{!}{%
    \begin{tabular}{lccccc} 
    \toprule
        \thead{Dataset} & \thead{Size} & \thead{Class\#} & \thead{Avg. \\ Node\#} & \thead{Avg. \\Edge\#} & \thead{Node\\ Label \#} \\ 
        \midrule
        KKI & 83 & 2 & 26.96 & 48.42 & 190 \\ 
        MUTAG & 188 & 2 & 17.93 & 19.79 & 7 \\
        PTC\_MM & 336 & 2 & 13.97 & 14.32 & 20 \\ 
        PTC\_MR & 344 & 2 & 14.29 & 14.69 & 18 \\ 
        PTC\_FM & 349 & 2 & 14.11 & 14.48 & 18 \\ 
        PTC\_FR & 351 & 2 & 14.56 & 15.00 & 19 \\
        BZR & 405 & 2 & 35.75 & 38.36 & 10 \\ 
        COX2 & 467 & 2 & 41.22 & 43.45 & 8 \\ 
        DHFR & 467 & 2 & 42.43 & 44.54 & 9 \\ 
        % NCI109 & 4127 & 2 & 29.68 & 32.13 & 38 \\ 
        ENZYMES & 600 & 6 & 32.63 & 62.14 & 3 \\ 
        PROTEINS & 1113 & 2 & 39.06 & 72.82 & 3 \\ 
        DD & 1178 & 2 & 284.32 & 715.66 & 82 \\ 
        NCI1 & 4110 & 2 & 29.87 & 32.30 & 37 \\ 
        IMDB-B& 1000& 2& 19.77& 96.53&-\\
        IMDB-M& 1500& 3& 13.00& 65.94&-\\
        % REDDIT-B& 2000& 2& 429.63& 497.75&-\\
        % REDDIT-M(5k)& 4999& 5& 508.52& 594.87&-\\
        COLLAB& 5000& 3& 74.49& 2457.78&-\\
        \bottomrule
    \end{tabular}
    }
    \caption{Statistics of the datasets we use. ``-" means the dataset has no node labels.}
    \label{table:datasets}
\end{table}

\label{sec:detailsOfDatasets}

\subsection{Details of NLM}
\label{sec:detailsOfNLM}

During the fine-tuning process of BERT, we treat the labels in slices as words, and the encoding of the whole slices as sentences. Subsequently, each label in the sequence is replaced with [MASK] with a certain probability \(p\). Using cross-entropy loss, we fine-tune the original pertained BERT model from \cite{BERT} (bert-base-uncased) to predict the label with bidirectional contexts. The embeddings of the last layer in BERT are extracted to obtain the embedding for each label. The dimensions of the embeddings for labels and slices are both set to 792, and the maximum token length for input is configured to be either 250 or 500, depending on the dataset size.

In training the Word2Vec (w2v) model \cite{word2vec}, the window size is set to 5, and the embedding size is set to 32.

\section{Further Experimental Evaluation Results}

\subsection{Parameter Sensitivity for Experiment Times}
\label{sec:Tsensitivity}
Figure \ref{fig:paramsensT} shows the average classification accuracy of DHGAK-BERT in 10-fold cross-validation with parameter \(T\in\{1,3,5,7,9\}\) and other parameters fixed on several datasets.
We can observe that DHGAK-BERT behaves stably over different values \(T\) and consistently achieves high performance.

\begin{figure}[!htbp]
    \centering
    \resizebox{0.9\linewidth}{!}{
    \includegraphics{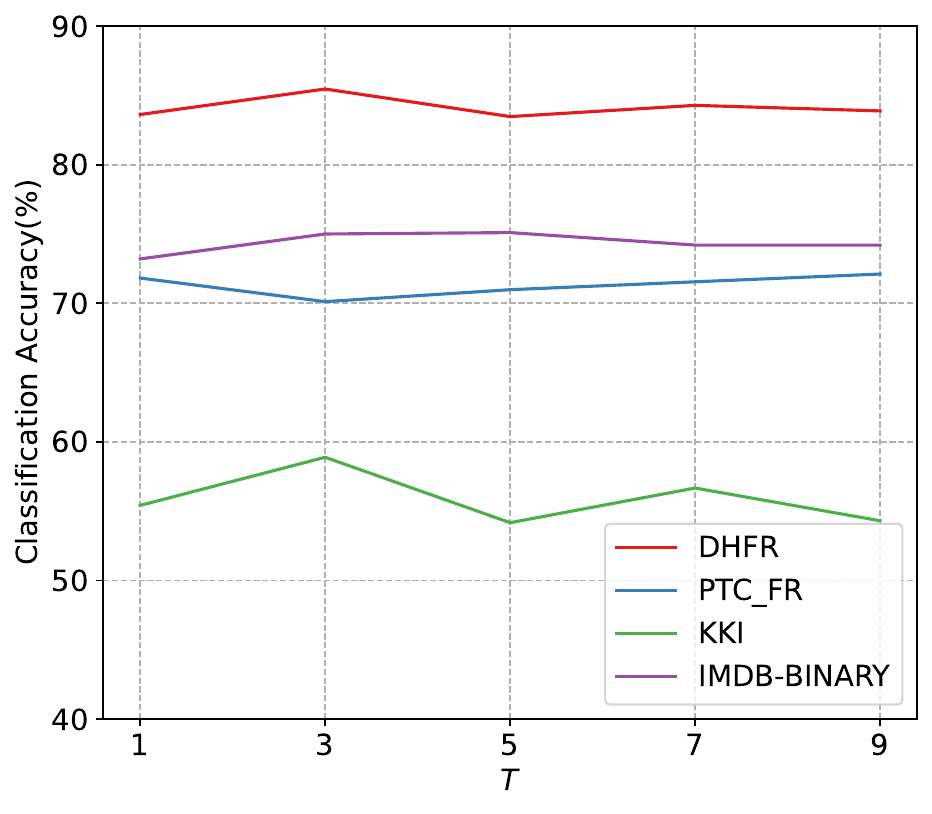}
    }
    \caption{Parameter sensitivity analysis on \(T\) in DHGAK-BERT.}
    \label{fig:paramsensT}
\end{figure}

\subsection{Additional Realizations of DHGAK}
\label{sec:additionReals}

Table \ref{tab:exReal} presents the average classification accuracy and variance of two additional realizations of DHGAK on various datasets: DHGAK-w2v (DBSCAN) utilizing DBSCAN \cite{dbscan} and DHGAK-w2v (K-means, DBSCAN) employing K-means and DBSCAN. Since the number of clusters is obtained automatically in DBSCAN, there is no need to grid search the parameter \(cluster\_factor\) for DBSCAN clustering method, and other parameters are grid-searched as the same as mentioned in Table \ref{tab:paramSearch}. The results indicate that both of the two additional realizations achieve competitive performances compared to GAWL \cite{GAWL}, highlighting the superior generalization of our proposed DHGAK.

\begin{table}[!htbp]
    \centering
    \resizebox{\linewidth}{!}{
    \begin{tabular}{l|cc|c}
        \toprule
         Datasets& \thead{DHGAK-w2v\\(DBSCAN)}&\thead{DHGAK-w2v\\(K-means, DBSCAN)}&GAWL\\
         \midrule
         KKI&  $55.4\pm5.2$&   $56.4\pm9.0$&$\mathbf{56.8\pm23.1}$\\
         MUTAG&  $88.9\pm4.9$&  $\mathbf{88.8\pm4.4}$&$87.3\pm 6.3^\dag$\\
         PTC\_MM&  $68.5\pm3.8$&  $\mathbf{69.3\pm7.6}$&$66.3\pm5.3 $\\
         PTC\_FR&  $68.7\pm3.0$&  $\mathbf{70.1\pm6.3}$&$64.7\pm3.6$\\
         DHFR&  $84.3\pm3.9$&  $83.5\pm3.7$ &$\mathbf{83.7\pm3.0}$\\
         ENZYMES&  $49.8\pm7.0$&   $\mathbf{58.5\pm7.5}$&$\mathbf{58.5\pm4.8}$\\
         IMDB-B&  $71.4\pm3.5$&   $\mathbf{76.0\pm2.0}$&$74.5\pm4.1^\dag$
\\
         IMDB-M& $48.3\pm2.7$&  $49.1\pm2.6$&$\mathbf{51.7\pm 5.2^\dag}$\\
         COLLAB& $78.3\pm1.2$&  $\mathbf{81.6\pm1.3}$&$81.5\pm2.0^\dag$\\
         \bottomrule
    \end{tabular}
    }
    \caption{Comparison of the average classification accuracy of DHGAK-w2v (DBSCAN), DHGAK-w2v (K-means, DBSCAN), and GAWL over the 10-fold cross-validation.}
    \label{tab:exReal}
\end{table}

\end{document}